\documentclass[letterpaper,journal]{IEEEtran}
\usepackage{amsmath,amsfonts}
\usepackage{algorithmic}
\usepackage{array}
\usepackage{textcomp}
\usepackage{stfloats}
\usepackage{url}
\usepackage{verbatim}
\usepackage{graphicx}

\usepackage{tocloft}
\usepackage{xcolor}
\usepackage{enumitem}
\usepackage{graphicx}
\usepackage{amsmath}
\usepackage{amssymb}
\usepackage{booktabs}

\usepackage{orcidlink}
\usepackage{ifthen}

\newboolean{showappendix}
\setboolean{showappendix}{false} 

\usepackage{colortbl}
\usepackage{xcolor}

\usepackage{bbm}
\usepackage{bm} %
\usepackage{amsthm}
\usepackage{float}
\usepackage{wrapfig}
\usepackage{subcaption}
\usepackage{xspace} 
\newcommand{\ie}{\emph{i.e.,}\xspace}
\newcommand{\eg}{\emph{e.g.,}\xspace}

\newcommand{\etal}{\emph{et~al.}\xspace} 
\newcommand{\aka}{\emph{a.k.a.,}\xspace}

\newcommand{\de}[1]{\textcolor{black}{ #1}} 
\newcommand{\bl}[1]{\textcolor{black}{ #1}} 
\newcommand{\cyan}[1]{\textcolor{black}{ #1}} 
 
\usepackage{graphicx}

\usepackage{hyperref}
\hypersetup{
    colorlinks=true,
    linkcolor=blue,
    filecolor=magenta,      
    urlcolor=cyan,
    pdftitle={Overleaf Example},
    pdfpagemode=FullScreen,
    }

\usepackage{xurl}
\theoremstyle{plain}
\newtheorem{theorem}{Theorem}

\newtheorem{corollary}{Corollary}
\theoremstyle{definition}
\newtheorem{definition}{Definition}

\theoremstyle{remark}

\def\BibTeX{{\rm B\kern-.05em{\sc i\kern-.025em b}\kern-.08em
    T\kern-.1667em\lower.7ex\hbox{E}\kern-.125emX}}
\usepackage{balance}

\begin{document}

\title{\bl{Generalized Regularized Evidential Deep Learning Models: Theory and Comprehensive Evaluation}}

\author{%
    Deep Shankar Pandey~\orcidlink{0009-0006-1404-3716
}$^{1}$, Hyomin Choi~\orcidlink{0000-0003-2458-893X}$^{2}$, and Qi Yu~\orcidlink{0000-0002-0426-5407}$^{1}$%
    \thanks{%
    $^{1}$ Deep Shankar Pandey and Qi Yu are with Rochester Institute of Technology, Rochester, NY, USA.\protect\\
    Email: \{dp7972, qyuvks\}@rit.edu%
    }%
    \thanks{%
    $^{2}$ Hyomin Choi is with AI Lab, InterDigital, CA, USA.\protect\\
    Email: hyomin.choi@interdigital.com%
    }%
    \thanks{%
    $^{2}$ \textbf{This work has been submitted to the IEEE for possible publication. Copyright may be transferred without notice, after which this version may no longer be accessible.}
    }%
}

\markboth{Journal of \LaTeX\ Class Files,~Vol.~18, No.~9, September~2020}%
{Generalized Regularized Evidential Deep Learning Models for Classification}

\maketitle

\begin{abstract}
Evidential deep learning (EDL) models, based on Subjective Logic, introduce a principled and computationally efficient way to make deterministic neural networks uncertainty-aware. The resulting evidential models can quantify fine-grained uncertainty using learned evidence. \bl{However, the Subjective-Logic framework constrains evidence to be non-negative, requiring specific activation functions whose geometric properties can induce activation-dependent learning-freeze behavior—a regime where gradients become extremely small for samples mapped into low-evidence regions. We theoretically characterize this behavior and analyze how different evidential activations influence learning dynamics. Building on this analysis, we design a general family of activation functions and corresponding evidential regularizers that provide an alternative pathway for consistent evidence updates across activation regimes. Extensive experiments on four benchmark classification problems (MNIST, CIFAR-10, CIFAR-100, and Tiny-ImageNet), two few-shot classification problems, and blind face restoration problem empirically validate the developed theory and demonstrate the effectiveness of the proposed generalized regularized evidential models.}
\end{abstract}


\section{Introduction}

With recent growth in computational capabilities, availability of large-scale data, and algorithmic improvements, Deep Learning (DL) models have found great success in many real-world applications such as speech recognition \cite{kamath2019deep}, machine translation \cite{singh2017machine}, and computer vision \cite{voulodimos2018deep}. 
However, these highly expressive models can easily fit the noise in the training data, leading to overconfident predictions \cite{nguyen2015deep}. This challenge is compounded in specialized domains  (\eg medicine, public safety, and military operations) where labeled data is limited and costly to obtain. Accurate uncertainty quantification is essential for the successful application of DL models in these domains. To this end, DL models have been augmented to become uncertainty-aware~\cite{gal2016dropout,blundell2015weight,pearce2020uncertainty}. However, commonly used extensions require expensive sampling operations \cite{gal2016dropout,blundell2015weight}, which significantly increase the computational costs \cite{lakshminarayanan2017simple}. 

The recently developed evidential deep learning (EDL) models bring together evidential theory~\cite{shafer1976mathematical,josang2016subjective} and deep neural architectures 
that turn a deterministic neural network uncertainty-aware. By leveraging the learned evidence, evidential models are capable of quantifying fine-grained uncertainty that helps to identify the sources of `unknowns'. Furthermore, since only lightweight modifications are introduced to existing DL architectures, additional computational costs remain minimal. Such evidential models have been successfully extended to classification \cite{sensoy2018evidential}, regression \cite{amini2020deep}, meta-learning \cite{Pandey_2022_CVPR}, and open-set recognition \cite{bao2021evidential} settings.

\begin{wrapfigure}{r}{0.26\textwidth}
\vspace{-4mm}
  \begin{center}
    \hspace{-0.7cm}\includegraphics[width=\linewidth]{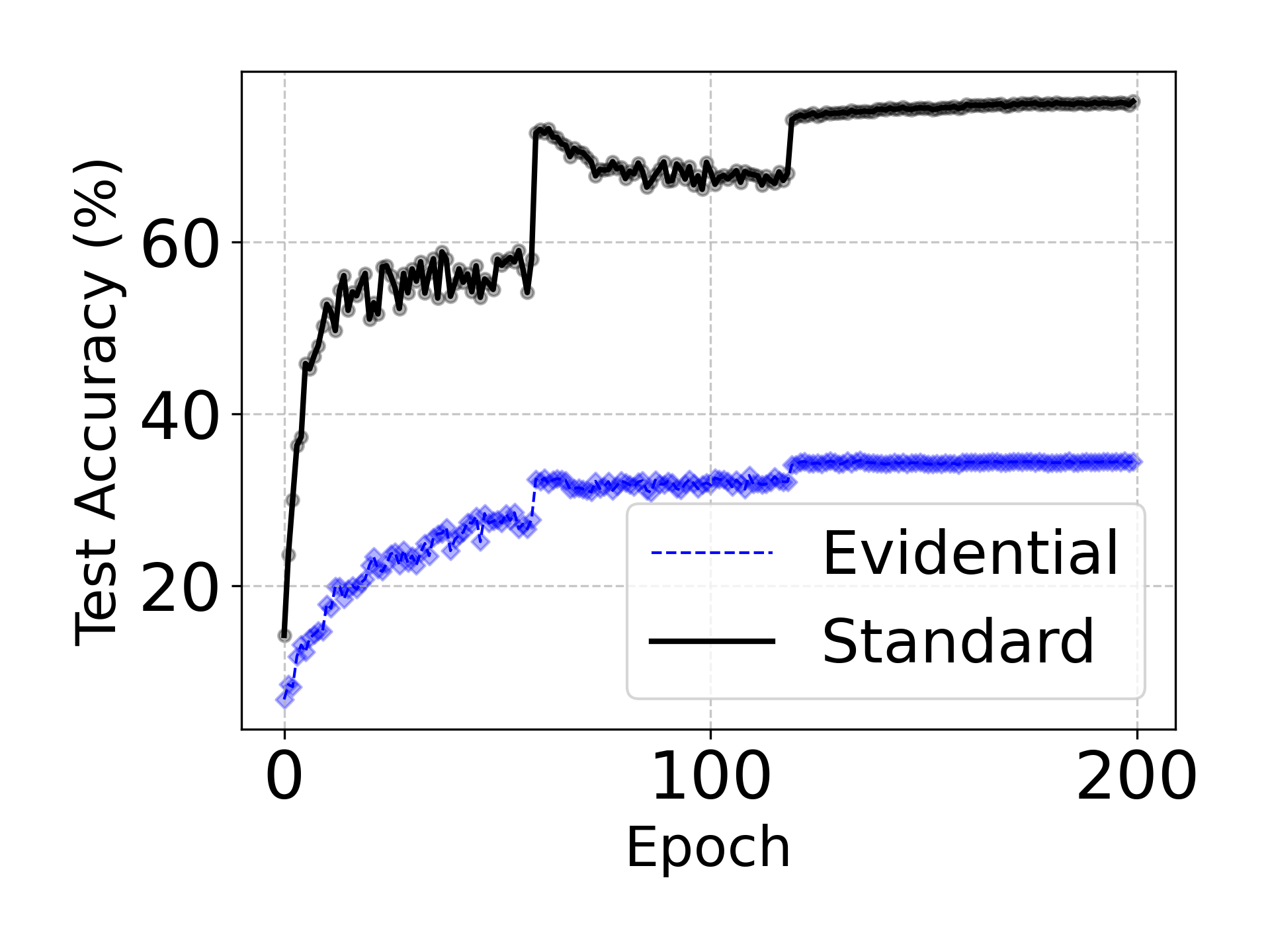}
  \end{center}
  \vspace{-4mm}
  \caption{Cifar-100 Result}
  \label{fig:cifar100MSEComp}
  \vspace{-2mm}
\end{wrapfigure}

Despite the attractive uncertainty quantification capacity, evidential models often achieve competitive predictive performance only on relatively simple learning problems. Their performance can degrade on more complex, large-scale datasets even in standard classification settings. As shown in Figure \ref{fig:cifar100MSEComp}, an evidential model using ReLU activation and an evidential MSE loss \cite{sensoy2018evidential} achieves around $36$\% test accuracy on CIFAR-100, nearly $40$\% lower than a standard softmax model. In addition, many evidential variants are sensitive to architecture or hyperparameter changes, requiring careful tuning for stable performance. The experiment section provides more details on these cases.

\bl{To better understand this phenomenon, we perform a theoretical analysis of evidential learning in the standard classification setting. Our results identify an \textit{activation-induced learning-freeze behavior}, where the interaction between non-negative evidence parameterization and specific activation functions can map samples into ``zero-evidence regions'' (regions of vanishing evidence gradients).}

\bl{Importantly, this behavior arises within the design choices of the EDL framework itself. EDL couples non-negative evidence parameterization with a KL-based prior that \textit{intentionally} promotes high epistemic uncertainty at class boundaries and in regions far from the training distribution—an effect that helps prevent overconfident errors. Activation functions and regularizers determine how evidence accumulates under this framework. Our analysis shows that commonly used non-negative activations can inadvertently create ``zero-evidence'' regions where gradients become extremely small, making evidence updates for nearby samples ineffective.}

\bl{More specifically, EDL models acquire limited new evidence from samples mapped into these low-evidence regions because the corresponding evidence gradients approach zero. Moreover, the learning signal decreases proportionally as samples are mapped closer to the zero-evidence region, irrespective of supervised information.}

\bl{This activation-induced stagnation is illustrated in Figure \ref{fig:intuitiveFailureOfRelu} (with detailed discussion in Section \ref{sec:evModelTrainingLoss}). We analyze several existing evidential variants and observe this behavior consistently across models and settings. Motivated by these insights, we introduce a novel \textbf{G}eneralized \textbf{R}egularized \textbf{E}vidential mo\textbf{d}el (\textbf{GRED}) that employs positive evidence regularization to encourage evidence accumulation even in low-evidence regimes.}

\bl{A preliminary version of this work has been published as a conference paper \cite{pandey2023learn}. Improving on RED, we propose generalized regularized evidential models that mitigate learning stagnation across a family of evidential activations (Section \ref{sec:newAvoidZeroEvReg}). We theoretically show the effectiveness of the correct-evidence regularization (Theorem \ref{th:sovlingzeroevidenceIssue}) and provide expanded analysis of evidential losses (Section \ref{ap:evLossAnalysis}). We further extend GRED to challenging few-shot classification and blind face restoration tasks, and carry out detailed uncertainty analysis, and demonstrate the broader utility of evidential uncertainty.}

\begin{figure}[t] 
\centering
  \includegraphics[width=0.95\linewidth]{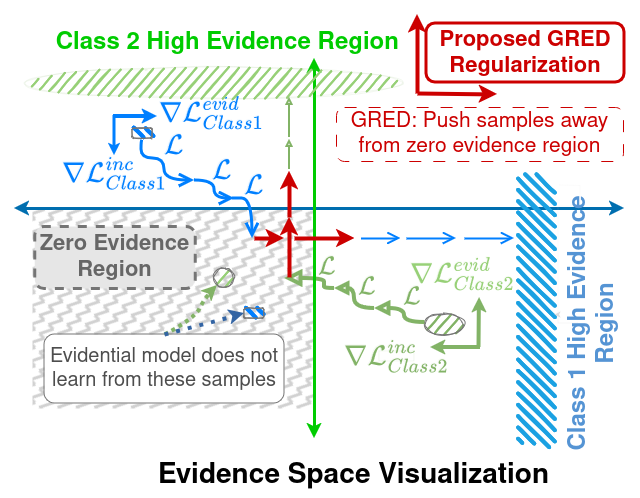}
\caption{\bl{Intuitive visualization of a zero-evidence region for evidential models in the evidence space for binary classification. Samples mapped into such regions have extremely small gradients, leading to limited model update during training. \bf{GRED encourages larger gradients for `zero-evidence' samples, enabling consistent learning across samples.}}}
\label{fig:intuitiveFailureOfRelu}
\end{figure} 
\bl{
Our major contributions can be summarized as follows:
\begin{itemize}[noitemsep,topsep=0pt,leftmargin=*]
    \item We identify an activation-induced learning-freeze behavior in evidential models, wherein data samples mapped to ``zero-evidence'' regions receive vanishing evidence gradients. For these samples, the learning signal decreases proportionally as they are mapped closer to the zero-evidence region in the evidence space.
    \item We theoretically show that evidential models with $\exp$ activation produce stronger gradients near low-evidence regions compared to other activations.
    \item We introduce a generalized evidence regularization strategy that encourages evidence updates across activation regimes, enabling more consistent learning from all samples.
    \item We conduct extensive experiments across multiple settings including 4 benchmark classification tasks, few-shot classification, and blind face restoration, validating the developed theory and demonstrating the effectiveness of our approach.
\end{itemize}
}

\section{Related Works}

\paragraph{Uncertainty Quantification in Deep Learning}
Accurate quantification of predictive uncertainty is essential for development of trustworthy Deep Learning (DL) models. 
To this end, DL models have been augmented to become uncertainty-aware \cite{hullermeier2021aleatoric} using a variety of approaches such as deep ensemble–based approaches \cite{lakshminarayanan2017simple, pearce2020uncertainty}, Bayesian neural networks \cite{mobiny2021dropconnect,gal2016dropout,blundell2015weight}, \cyan{second-order distribution–based approaches \cite{sale2024second, ren2021kronecker,sale2023second}, and credal-set–based approaches \cite{caprio2023credal, caprio2024credal, cozman2000credal, zaffalon2002naive, caprio2023novel, sale2023volume}}.
Deep ensemble techniques \cite{lakshminarayanan2017simple, pearce2020uncertainty} construct an ensemble of neural networks and the agreement/disagreement across the ensemble components is used to quantify different uncertainties. Ensemble-based methods significantly increase the number of model parameters, and are computationally expensive at both training and test times.

Bayesian neural networks  \cite{gal2016dropout}\cite{blundell2015weight}\cite{mobiny2021dropconnect} have been developed that consider a Bayesian formalism to quantify different uncertainties. For instance, Blundell \etal ~\cite{blundell2015weight} use Bayes-by-backdrop to learn a distribution over neural network parameters, whereas Gal \etal~\cite{gal2016dropout} enable dropout during inference phase to obtain predictive uncertainty. Bayesian methods resort to some form of approximation to address the intractability issue in marginalization of latent variables. Moreover, these methods are also computationally expensive as they require sampling for uncertainty quantification.

Towards accurate UQ, Credal Bayesian deep learning (CBDL) models \cite{caprio2023credal, caprio2024credal, ristic2020tutorial} have also been developed that use concepts from the imprecise probability theory \cite{augustin2014introduction} for comprehensive uncertainty quantification. These models aim to approximate the credal set of posterior distributions during training \cite{caprio2023credal}, based on which, the models infer the credal set of predictive distributions during inference. CBDL models are more robust to prior/likelihood distribution misspecification compared to BNN, have been effectively applied to continual learning settings \cite{lu2023ibcl}, and provide more robust epistemic uncertainty quantification capabilities \cite{caprio2024credal} especially in situations of prior/likelihood misspecification. However, these models, due to the use of credal sets that require reasoning over multiple distributions, are computationally expensive compared to BNN, limiting their usage. In contrast, Evidential DL approaches only require a single forward pass through the deep learning models to quantify uncertainty, making them computationally lighter compared to BNN and CBDL models. 

\paragraph{Evidential Deep Learning}
\bl{
Uncertainty in belief/evidence theory \cite{denoeux2020representations, denoeux2001handling} and its neural extensions \cite{denoeux2000neural, denoeux2022evidential} have been studied under Dempster–Shafer Theory \cite{shafer1992dempster}, fuzzy logic \cite{denoeux2023quantifying, de2018intelligent}, and Subjective Logic \cite{josang2016subjective}. Evidential deep learning is closely related to second-order distribution–based uncertainty quantification \cite{sale2024second, ren2021kronecker,sale2023second} and frequently employs Subjective Logic \cite{josang2016subjective} for uncertainty reasoning.
Evidential models introduce a conjugate higher-order evidential prior for the likelihood distribution that enables the model to capture the fine-grained uncertainties. For instance, Dirichlet prior is introduced over the multinomial likelihood for evidential classification \cite{bao2021evidential,zhao2020uncertainty,charpentier2020posterior}, and the normal-inverse-gamma prior is introduced over the Gaussian likelihood \cite{amini2020deep, pandey2022evidential} for the evidential regression models. }

The robustness \cite{kopetzki2021evaluating} and calibration \cite{tomani2021towards} of evidential models have been extensively studied. Usually, these models are trained with evidential losses and heuristically designed regularizers to guide uncertainty behavior \cite{Pandey_2022_CVPR, shi2020multifaceted}. Some variants incorporate out-of-distribution (OOD) data into training \cite{malinin2019reverse,malinin2018predictive}, but this assumption may not hold in real-world settings. A recent survey \cite{ulmer2021survey} provides a comprehensive overview.

\bl{Recent works have examined the reliability of uncertainty measures in EDL. Wimmer \etal~\cite{wimmer2023quantifying} and Bengs \etal~\cite{bengs2022pitfalls} highlight cases where the decomposition of aleatoric and epistemic uncertainty may be inconsistent. Jurgens \etal~\cite{jurgens2024epistemic} study epistemic uncertainty behavior and report situations where evidential uncertainty can be misleading. Shen \etal~\cite{shen2024uncertainty} extend EDL by explicitly modeling additional forms of model uncertainty, while second-order approaches \cite{sale2024second, sale2023second} propose theoretically grounded uncertainty measures addressing limitations in earlier formulations.}

\bl{Orthogonal to these existing evidential works—which primarily investigate the \emph{interpretation} or \emph{reliability} of evidential uncertainty—we study the \emph{training dynamics} of evidential models. Specifically, we characterize the \textit{activation-induced learning-freeze behavior}: certain non-negative evidence activations can map samples into low-evidence regions where gradients become extremely small, limiting effective evidence updates. We introduce a theoretically justified regularization strategy that mitigates this stagnation and enables more consistent evidence accumulation across activation regimes.}

In this work, we focus on evidential classification models and consider settings without access to the out-of-distribution data during training, improving applicability in real-world scenarios.

\section{Learning in Evidential Models}
\label{sec:EvDlModel}
\de{We first describe learning in standard classification models. We then describe evidential deep learning model basics for classification. Afterward, we analyze the gradient dynamics of evidential training to characterize the \textit{activation-induced learning-freeze} behavior that arises when certain non-negative evidence activations map samples into low-evidence regions.}
\subsection{Standard Classification Models} \label{sec:appAnalysisStandardClassificationModels}
Standard classification models use a softmax transformation on the output from the neural network $\mathcal{F}_{\Theta}$ for input $\mathbf{x}$ to obtain the class probabilities in a $K$-class classification problem. 
Such models are trained with cross-entropy-based losses. 
These models have achieved state-of-the-art performance on many benchmark problems. 
\paragraph{Gradient Analysis} Consider a standard cross-entropy trained model for $K-$class classification. Let the overall network be represented by $f_{\Theta}(\cdot)$, and let $\mathbf{o} = f_{\Theta}(\mathbf{x})$ be the output from this network before the softmax layer for input $\mathbf{x}$ and one-hot ground truth label of $\mathbf{y}$. The output at node $i$ after the softmax layer is given by 
\begin{align}
\texttt{sm}_i = \frac{\exp({o}_i)}{\sum_{k=1}^{K} \exp({o}_k)} = \frac{\exp({o}_i)}{S^{\text{ce}}}    
\end{align}
\noindent where $S^{\text{ce}} =  \sum_{k=1}^K \exp(o_k)$. For a given sample $(\mathbf{x}, \mathbf{y})$, the cross-entropy loss ($\mathcal{L}_{\texttt{ce}}$), and the gradient of this loss with respect to the pre-softmax values $\mathbf{o}$ are given by
\begin{align}
    \mathcal{L}_{\texttt{ce}} &= -\sum_{k = 1}^K y_k \log (\texttt{sm}_k) \quad
    = \log S^{\text{ce}} - \sum_{k=1}^K y_k o_k
\end{align}
\begin{align}
    \text{grad}_k &= \frac{\partial \mathcal{L}_{\texttt{ce}}}{\partial o_k} = \Big( \frac{1}{S^{\text{ce}}}\frac{\partial S^{\text{ce}}}{\partial o_k} - y_k \Big) 
    = \frac{\exp (o_k)}{S^{\text{ce}}} - y_k  \\
    &= \texttt{sm}_k - y_k
\end{align}
The gradient measures the error signal, and for standard classification models, it is bounded in the range [-1, 1] as $0 \leq \texttt{sm}_k \leq 1$ and $y_k \in \{0, 1\}$. The model is updated using gradient descent-based optimization objectives. For input $\mathbf{x}$, the neural network outputs K values $o_1$ to  $o_K$, and the corresponding ground truth is $\mathbf{y}$, $y_{\text{gt}} = 1, y_{\neq gt} = 0$. When $y_i$ = 0, the gradient signal is $\text{grad}_i = \texttt{sm}_i$ and the model optimizes the parameters to minimize this value. Only when $\texttt{sm}_i = 0$, the gradient is zero, and the model is not updated. In all other cases when $\texttt{sm}_i \neq 0$, there is a non-zero gradient dependent on $\texttt{sm}_i$, and the model is updated to minimize the $\texttt{sm}_i$ as expected. When $y_i$ = 1, the gradient signal is $\text{grad}_i = \texttt{sm}_i - 1$ and the model optimizes the parameters to minimize this value. As $ \texttt{sm}_i \in [0, 1]$, only when the model outputs a large logit on $i$ (corresponding to the ground-truth class) and small logit for all other nodes, $\texttt{sm}_i = 1$, the gradient is zero, and the model is not updated. For the cases when $\texttt{sm}_i < 1$, there is a non-zero gradient dependent on $\texttt{sm}_i$ and the model is updated to maximize the $\texttt{sm}_{i=gt}$ and minimize all other $\texttt{sm}_{i\neq gt}$ as expected. The gradient signal in standard classification models trained with standard cross-entropy loss is reasonable and enables learning from all the training samples. 

The above gradient analysis shows that standard classification models trained with cross-entropy-based loss can effectively learn from all the training samples. Nevertheless, these models lack a systematic mechanism to quantify different sources of uncertainty, a highly desired property in many real-world problems.

\subsection{Evidential Deep Learning Models for Classification}
\begin{figure}[htpb]
\centering
  \includegraphics[width=0.9\linewidth]{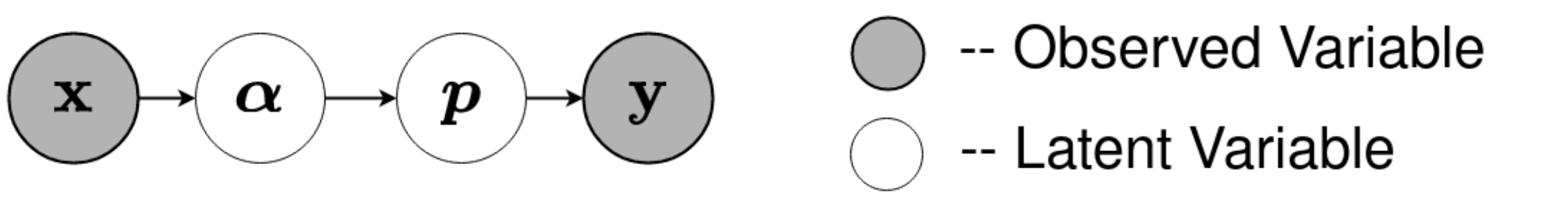}
  \caption{Graphical model for Evidential Models } 
\label{fig:evidGraphicalRep}
\centering
\end{figure}
Evidential deep learning models for classification formulate model training as an evidence acquisition process and consider a higher-order Dirichlet prior $\texttt{Dir}(\mathbf{p}|\boldsymbol{\alpha})$ over the predictive Multinomial distribution $\texttt{Mult}(\mathbf{y}|\mathbf{p})$. Different from a standard Bayesian formulation which optimizes {\em Type-II Maximum Likelihood} to learn the Dirichlet hyperparameter~\cite{bishop2006pattern}, evidential models directly predict $\boldsymbol{\alpha}$ using data features ${\bf x}$ and then generate the prediction ${\bf y}$ by marginalizing the Multinomial parameter ${\bf p}$. Figure \ref{fig:evidGraphicalRep} describes this generative process. Such higher-order prior enables the model to systematically quantify different sources of uncertainty. It is worth noting that the uncertainty behavior of vanilla EDL \cite{sensoy2018evidential} in low-evidence or boundary regions is an intentional design choice arising from its Dirichlet prior and KL-based regularization, which promote conservative (high epistemic) uncertainty away from training data. In evidential models, the \texttt{Softmax} in the standard neural networks for classification is replaced by a non-negative activation function $\mathcal{A}$, where $\mathcal{A}({\bf x}) \geq 0 \quad \forall x \in [-\infty, \infty]$, such that for input $\mathbf{x}$, the neural network model $\mathcal{F}_{\Theta}$ with parameters $\Theta$ can output evidence $\mathbf{e}$ for different classes. Dirichlet parameter $\boldsymbol{\alpha}$ is evaluated as $\boldsymbol{\alpha} = \mathbf{e} + \boldsymbol{1}$ to ensure $\boldsymbol{\alpha} \geq 1$, where $\mathbf{e} = \mathcal{A}(\mathcal{F}_\Theta (\mathbf{x})) = \mathcal{A}(\mathbf{o})$, which quantify fine-grained uncertainties in addition to the prediction  $\mathbf{y}$ for the input $\mathbf{x}$. Then, Dirichlet strength, $S$, for $K-$class classification problem is computed by
\begin{align}
    &S =  \sum_{k=1}^K (\mathbf{e}_k +1)
\end{align}
The activation function $\mathcal{A}(\cdot)$ assumes three common forms to transform the output into evidence: (1) $\texttt{ReLU}(\cdot) = \max( 0, \cdot )$, (2) $\texttt{SoftPlus}(\cdot) = \log ( 1 + \exp(\cdot) )$, and (3) $\exp(\cdot)$. 

Evidential models assign input sample to that class for which the output evidence is greatest. Moreover, they quantify the confidence in the prediction for $K-$class classification problem through vacuity $\nu = \frac{K}{S}$ (\ie measure of lack of confidence in the prediction.)
For any training sample $(\mathbf{x}, \mathbf{y})$, the evidential models aim to maximize the evidence for the correct class, minimize the evidence for the incorrect classes, and output accurate confidence. To this end, three variants of evidential loss functions have been proposed~\cite{sensoy2018evidential}: 1) Bayes risk with the sum of squares loss, 2) Bayes risk with cross-entropy loss, and 3) Type-II Maximum Likelihood loss. Please refer to Eq.~\ref{eqn:evMSEloss}, Eq.~\ref{eqn:evDigammaloss}, and Eq.~\ref{eqn:evLogloss}  in the Appendix for the specific forms of these losses. Additionally, incorrect evidence regularization terms are introduced to guide the model to output low evidence for classes other than the ground truth class (See Appendix \ref{app:evIncReg} for discussion on the regularization). With evidential training, evidential deep learning models are expected to output high evidence for the correct class, low evidence for all other classes, and output high vacuity for unseen/OOD samples. 

\subsection{Theoretical Analysis of Evidential Classification Models}

To identify the underlying reason that causes the performance gap of evidential models as described earlier, we consider a $K-$class classification problem and a representative evidential model trained using Bayes risk with the sum of squares loss given in Eq.~\ref{eqn:evMSEloss}. We first provide important definitions that are critical for our theoretical analysis.  
\begin{definition}[{\bf Zero Evidence Sample and Zero Evidence Region}] 
    A \textit{zero evidence sample} is a data sample for which the model outputs zero evidence for all classes. A \textit{zero evidence region} is the area in the evidence space that contains all the \textit{zero evidence samples}.
\end{definition}
For a reasonable evidential model, novel data samples not yet seen during training, difficult data samples, and Out-Of-Distribution (OOD) samples should become \textit{zero evidence samples} and should be mapped in the \textit{zero evidence region}. 

\begin{theorem}
\label{supotimalityTheorem}
{Given a training sample $(\mathbf{x}, \mathbf{y})$, if an evidential neural network outputs zero evidence $\mathbf{e}$, then the gradients of the evidential loss evaluated on this training sample over the network parameters reduce to zero. }
\end{theorem}
\begin{proof}
Consider an input $\mathbf{x}$ with one-hot ground truth label $\mathbf{y}$. Let the ground truth class index be $gt$, \ie $y_{\text{gt}} =1, $ with corresponding Dirichlet parameter $\alpha_{\text{gt}}$, and $y_{\neq gt} = 0$. Moreover, let $\mathbf{o}, \mathbf{e}, \text{and } \boldsymbol{\alpha}$ represent the neural network output vector before applying the activation $\mathcal{A}$, the evidence vector, and the Dirichlet parameters respectively. 

In this evidential model, the loss is given by
\begin{align}
    \mathcal{L}^{\texttt{MSE}}(\mathbf{x}, \mathbf{y}) &= \sum_{j=1}^K (y_j - \frac{\alpha_j}{S})^2 + \frac{\alpha_j (S - \alpha_j)}{S^2(S+1)} 
\end{align}

Now, the gradient of the loss with respect to the  neural network output can be computed using the chain rule:
\begin{align}\label{eqn:gradMseInt}
\begin{split}
    &\frac{\partial \mathcal{L}^{\texttt{MSE}}(\mathbf{x}, \mathbf{y})}{\partial o_k}  = \frac{\partial \mathcal{L}^{\texttt{MSE}}(\mathbf{x}, \mathbf{y})}{\partial \alpha_k}\frac{\partial e_k}{\partial o_k} \\
    & = \bigg[ \frac{2\alpha_{\text{gt}}}{S^2} - 2\frac{y_k}{S} - \frac{2( S - \alpha_k)}{S(S+1)} +\\
    &\quad\quad\quad+\frac{2(2S + 1)\sum_{i} \sum_{j}\alpha_i \alpha_j}{(S^2 + S)^2}
    \bigg] \times \frac{\partial e_k}{\partial o_k}
    \end{split}
\end{align}
Based on the actual form of $\mathcal{A}$, we have three cases:

\begin{itemize}
    \item \textbf{Case I:} $\texttt{ReLU}(\cdot)$ to transform logits to evidence
    \begin{align}\label{eqn:reluGradoE}
    e_k &= \texttt{ReLU}(o_k) 
    \implies \frac{\partial e_k}{\partial o_k} = \begin{cases}
    1 \quad  $\text{if }$ o_k > 0 \\
    0 \quad  \text{otherwise}
    \end{cases}    
    \end{align}
    For a zero evidence sample, the logits $o_k$ satisfy the relationship $o_k \leq 0 \; \forall \; k
    \implies \frac{\partial e_k}{\partial o_k} = 0
    \implies\frac{\partial \mathcal{L}^{\texttt{MSE}}(\mathbf{x}, \mathbf{y})}{\partial o_k}  = 0
    $
    \item \textbf{Case II:} $\texttt{SoftPlus}(\cdot)$ to transform logits to evidence
    \begin{align}\label{eqn:SoftPlusGradoE}
    e_k = \texttt{Softplus}(o_k) \implies \frac{\partial e_k}{\partial o_k} = \texttt{Sigmoid}(o_k) 
    \end{align}
    For a zero evidence sample, the logits $o_k \rightarrow -\infty \implies \texttt{Sigmoid}(o_k) \rightarrow 0 \; \& \; \frac{\partial e_k}{\partial o_k} \rightarrow 0$.
    \item \textbf{Case III:} $\exp(\cdot)$ to transform logits to evidence
    \begin{align}\label{eqn:expGradoE}
        e_k &= \exp(o_k)
        \implies &\frac{\partial e_k}{\partial o_k} = \exp(o_k) = \alpha_k - 1
    \end{align}
     For a zero evidence sample, $\alpha_k \rightarrow 1 \implies \frac{\partial e_k}{\partial o_k} \rightarrow 0$. 
\end{itemize}
So, for all three instances of the evidential activation, $\frac{\partial e_k}{\partial o_k} \rightarrow 0 \text{ as } e_k \rightarrow 0 \quad \& \quad e_k = 0 \implies \frac{\partial e_k}{\partial o_k} = 0$. Moreover, there is no term in the first part of the loss gradient in Eq.~\ref{eqn:gradMseInt} to counterbalance these zero-approaching gradients. \de{As the training sample is mapped to the region near the zero evidence region (\ie $e_k \rightarrow 0$), the evidence gradients ($\frac{\partial e_k}{\partial o_k}$) approach to zero (\ie $\frac{\partial e_k}{\partial o_k} \rightarrow 0 $), and the loss gradient (\aka the learning signal) also approaches zero (\ie $\frac{\partial \mathcal{L}^{\texttt{MSE}}(\mathbf{x}, \mathbf{y})}{\partial o_k}  \rightarrow 0$) irrespective of the supervised learning signal.} For \textit{zero evidence training samples}, for any node $k$,
\begin{align}
\frac{\partial \mathcal{L}^{\texttt{MSE}}(\mathbf{x}, \mathbf{y})}{\partial o_k}  = 0
\end{align}
For \textit{zero evidence training samples}, since the gradient of the loss with respect to all the nodes is zero, there is no update to the model from such samples. {This \de{shows} that the evidential models fail to learn from a zero evidence data sample.} 
\end{proof}
For completeness, we present the detailed proof of the evidential models trained using Bayes risk with the sum of squares error along with other evidential losses in Appendix \ref{apSec:evidentialProof}, and impact of incorrect evidence regularization in Appendix \ref{app:evIncReg}. 

\paragraph{Remark}
\bl{When an evidential model outputs zero evidence for all classes (\ie a data sample that the model has never seen and for which the model accurately outputs ``I don't know'', \ie  $e_k=0$ $\forall k \in [1, K]$), the gradients of standard evidential losses vanish, and the supervised information in such samples cannot contribute to parameter updates.
Such samples may naturally appear during training (for example, novel, ambiguous, or OOD-like inputs), but the model receives no learning signal from them because they lie in the \textit{zero evidence region}. 
Similarly, samples mapped near the \textit{zero evidence region} receive significantly diminished gradients: their learning signal becomes much weaker than that of samples with higher evidence, regardless of the strength of the supervised label.}

\begin{corollary}
    \bl{Incorrect evidence regularization does not provide a learning signal for zero evidence samples and therefore cannot induce parameter updates for such samples.}  
\end{corollary}
\bl{Intuitively, the incorrect evidence regularization encourages the model to reduce evidence for non–ground-truth classes, but it does not increase the evidence of the ground-truth class. As a result, its gradients do not affect the zero-evidence condition. Consequently, the regularization can move samples closer to the ``zero evidence region'' in the evidence space, but it cannot create a non-zero gradient for samples already mapped to this region. Thus, incorrect evidence regularization does not supply the missing gradient needed for zero-evidence samples to contribute to learning.}

\begin{theorem}
\label{th:superirorityofExp}
{ For a data sample $\mathbf{x}$, if an evidential model outputs logits $\mathbf{o}_k \leq 0 $ $\forall k \in [0, K]$, the exponential activation function leads to a larger gradient update on the model parameters than \texttt{SoftPlus} and \texttt{ReLu}.}
\begin{proof}
{
Consider an evidential loss $\mathcal{L}$ (formally defined in Eq.~\ref{eqn:evMSEloss}, Eq.~\ref{eqn:evDigammaloss}, and Eq.~\ref{eqn:evLogloss}) is used to train the evidential model. Let ${\bf o, e} \in \mathbb{R}^K$ denote the neural network output vector before applying the activation $\mathcal{A}$, and the evidence vector, respectively, for a network with weight $w$. For a data sample {$\bf x$}, if the network outputs $o_k<0, \forall k \in [K]$, we have: \\}\\
\noindent 1. \texttt{ReLU}: 
\begin{align}
    \Big(\frac{\partial \mathcal{L}}{\partial w}\Big)_{\texttt{ReLU}} =\sum_k\frac{\partial \mathcal{L}}{\partial e_k}\frac{\partial e_k}{\partial o_k}\frac{\partial o_k}{\partial w} = 0  \quad \quad \text{(see Eq.~\ref{eqn:reluGradoE})}
\end{align}
2. \texttt{SoftPlus}: 
\begin{align}
    \Big(\frac{\partial \mathcal{L}}{\partial w}\Big)_{\texttt{SoftPlus}}  &=\sum_k\frac{\partial \mathcal{L}}{\partial e_k}\frac{\partial e_k}{\partial o_k}\frac{\partial o_k}{\partial w} \quad \quad \text{(see Eq.~\ref{eqn:SoftPlusGradoE})}\\
    &=\sum_k\frac{\partial \mathcal{L}}{\partial e_k} \frac{\partial o_k}{\partial w} \texttt{Sigmoid}(o_k)
\end{align}
3. Exponential: 
\begin{align}
    \Big(\frac{\partial \mathcal{L}}{\partial w} & \Big)_{\text{Exp}}=\sum_k\frac{\partial \mathcal{L}}{\partial e_k}\frac{\partial e_k}{\partial o_k}\frac{\partial o_k}{\partial w} \quad \quad \text{(see Eq.~\ref{eqn:expGradoE})}\\
    &= \sum_k\frac{\partial \mathcal{L}}{\partial e_k} \frac{\partial o_k}{\partial w} \exp(o_k) \\
    &=\sum_k\frac{\partial \mathcal{L}}{\partial e_k} \frac{\partial o_k}{\partial w} \{[1 + \exp(o_k)]\texttt{Sigmoid}(o_k)\} 
\end{align}
$$
$$
{Thus, we have $\Big(\frac{\partial \mathcal{L}}{\partial w} \Big)_{\text{Exp}}\geq\Big(\frac{\partial \mathcal{L}}{\partial w}\Big)_{\texttt{SoftPlus}}\geq\Big(\frac{\partial \mathcal{L}}{\partial w}\Big)_{\texttt{ReLU}}$, which implies that $\mathcal{A}=\exp$ leads to a larger update to the network than both \texttt{SoftPlus} and \texttt{ReLU}. This completes the proof. }
\end{proof}
\end{theorem}
\paragraph{Remark} The above proof implies that the training of evidential models is most effective with the exponential activation function \de{as it has larger gradient update (and effectively stronger learning signal) for points near the \textit{zero evidence region} \ie for points with $\mathbf{o}_k \leq 0 $ $\forall k \in [0, K]$}. 

We now carry out additional analysis with a representative evidential model in $K-$class classification problem. We consider an input $\mathbf{x}$ with one-hot label of $\mathbf{y}$, $\sum_{k=1}^K y_k = 1$. For this evidential framework, the Type-II Maximum Likelihood loss ($\mathcal{L}^{\texttt{Log}}(\mathbf{x}, \mathbf{y})$) and its gradient with the logits $\mathbf{o}$ (Eq.~\ref{eq:gradtype2}) are given by
\begin{align}
     \mathcal{L}^{\texttt{Log}}(\mathbf{x}, \mathbf{y}) 
     &=  \log S - \sum_{k=1}^K y_k \log \alpha_k \\
    \text{grad}_k &= \frac{\partial \mathcal{L}^{\texttt{Log}}(\mathbf{x}, \mathbf{y}) }{\partial o_k}
    = \Big(\frac{1}{S} - \frac{y_k}{\alpha_k} \Big)\frac{\partial e_k}{\partial o_k}
\end{align}
\textbf{Case I and II:} $\texttt{ReLU}$$(\cdot)$ and $\texttt{SoftPlus}$$(\cdot)$ to transform logits to evidence.
\begin{itemize}
    \item \textbf{Zero evidence region:} For $\texttt{ReLU}$$(\cdot)$ based evidential models, if the logits value for class $k$ \ie $o_k$ is negative, then the corresponding evidence for class $k$ \ie $e_k = 0$, $\frac{\partial e_k}{\partial o_k} = 0 
    \; \& \; \text{grad}_k = \frac{\partial \mathcal{L}^{\texttt{Log}}(\mathbf{x}, \mathbf{y})}{\partial o_k}  = 0$. So, there is no update to the model through the nodes that output negative logits value.  
    In the case of $\texttt{SoftPlus}$$(\cdot)$ based evidential models, there is no update to the model when training samples lie in zero evidence regions. This is possible in the condition of $o_k \rightarrow - \infty$. In other cases, there will be some small finite small update in the accurate direction from the gradient. 
    \item \textbf{Range of gradients:} The range of gradients for both $\texttt{ReLU}$$(\cdot)$ and $\texttt{SoftPlus}$$(\cdot)$ based evidential models are identical. Considering the gradient for the ground truth node \ie $ y_k = 1$, the range of gradients is $[\frac{1}{K}-1, 0]$. For all other nodes other than the ground truth node \ie $y_k = 0$, the range of gradients is $[0, \frac{1}{K}]$. So, for classification problems with a large number of classes, the gradient updates to the nodes that do not correspond to the ground truth class will be bounded in a small range and is likely to be very small.
    \item \textbf{High incorrect evidence region:} If the evidence for class $k$ is very large \ie $e_k \rightarrow \infty$, then for $\texttt{ReLU}$$(\cdot)$, $\frac{\partial e_k}{o_k} = 1$,  and for $\texttt{SoftPlus}$$(\cdot)$, $\frac{\partial e_k}{o_k} = \texttt{Sigmoid}(o_k) \rightarrow 1, \frac{1}{\alpha_k} = \frac{1}{e_k + 1} \rightarrow 0, \frac{1}{S} \rightarrow 0, \; \& \; \text{grad}_k = \frac{\partial \mathcal{L}^{\texttt{Log}}(\mathbf{x}, \mathbf{y})}{\partial o_k}  \rightarrow 0$. For large positive model evidence, there is no update to the corresponding node of the neural network. The evidence can be further broken down into correct evidence (corresponding to the evidence for the ground truth class), and incorrect evidence (corresponding to the evidence for any other class other than the ground truth class). When the correct class evidence is large, the corresponding gradient is close to zero and there is no update to the model parameters which is desired. When the incorrect evidence is large, the model should be updated to minimize such incorrect evidence. However, the evidential models with $\texttt{ReLU}$ and $\texttt{SoftPlus}$ fail to minimize incorrect evidence when the incorrect evidence value is large. These necessities the need for incorrect evidence regularization terms. 
\end{itemize}

\textbf{Case III:} Exponential, $\exp$$(\cdot)$, to transform logits to evidence. Considering Eq.~\ref{eq:gradtype2} and Eq.~\ref{eqn:gradExpActivation}, the gradient of the loss with respect to the logits becomes 
\begin{align} 
    \text{grad}_k = 
    \frac{\partial \mathcal{L}^{\texttt{Log}}(\mathbf{x}, \mathbf{y})}{\partial o_k} 
    = \Big(\frac{1}{S} - \frac{y_k}{\alpha_k}\Big) (\alpha_k - 1)
\end{align}
 
\begin{itemize}
    \item \textbf{Zero evidence region:} In case of $\exp$$(\cdot)$ based evidential models, except in the extreme cases of $\alpha_k \rightarrow \infty$, there will be some signal to guide the model.
    In cases outside the zero evidence region (\ie outside $\alpha_k \rightarrow \infty$), there will be some finite small update in the accurate direction from the gradient. Moreover, for same evidence values, the gradient of $\exp$ based model is larger than the $\texttt{SoftPlus}$ based evidential model by a factor of $1 + \exp(o_k)$. Compared to $\texttt{SoftPlus}$ models, the larger gradient is expected to help the model learn faster in low-evidence regions. 
    \item \textbf{Range of gradients:} For the ground truth node, \ie $, y_k = 1$, the range of gradients is $[-1, 0]$. For all nodes other than the ground truth node \ie, $y_k = 0$, the range of gradients is $[0, 1]$. Thus, the gradients are expected to be more expressive and accurate in guiding the evidential model compared to $\texttt{ReLU}$ and $\texttt{SoftPlus}$ based evidential models.
    \item \textbf{High evidence region:} If the evidence for class $k$ is very high \ie $e_k \rightarrow \infty$, then $\alpha_k - 1 \thickapprox \alpha_k $ and $\text{grad}_k = \texttt{sm}_k - y_k$. In other words, the model's gradient updates become identical to the standard classification model (see Section \ref{sec:appAnalysisStandardClassificationModels}) without any learning issues.
\end{itemize}
Due to smaller zero evidence region, more expressive gradients, and no issue of learning in high incorrect evidence region, the exponential-based evidential models { are expected to be more effective }compared to $\texttt{ReLU}$ and $\texttt{SoftPlus}$ based evidential models.  
As can be seen, the $\texttt{ReLU}$ based activation completely destroys all the information in the negative logits and has the largest region in evidence space in which training data have zero evidence. $\texttt{SoftPlus}$ activation improves over the $\texttt{ReLU}$, and compared to $\texttt{ReLU}$, has a smaller region in evidence space where training data have zero evidence. However, \texttt{SoftPlus} based evidential models fail to correct the acquired knowledge when the model has strong wrong evidence. Moreover, these models are likely to suffer from the vanishing gradients problem when the number of classes increases (\ie classification problem becomes more challenging). Finally, exponential activation has the smallest zero evidence region in the evidence space without suffering from the issues of $\texttt{SoftPlus}$ based evidential models. \de{Still, the learning signal for all evidential models reduces proportionally as the training data points become closer to \textit{zero evidence region}, and the learning signal becomes zero for samples in \textit{zero evidence region} of the evidence space irrespective of the supervised signal in the training data point. This problem exists for all the activation functions.}

\section{Avoiding Zero Evidence Regions Through Correct Evidence Regularization}
\label{sec:newAvoidZeroEvReg}
\bl{We introduce a generalized \emph{correct evidence regularization} for evidential classification models that provides a meaningful gradient for samples in low- or zero-evidence regions, while leaving standard evidential losses unchanged for high-evidence samples.}

\subsection{Correct Evidence Regularization}
\bl{As shown in Section~\ref{sec:appAnalysisStandardClassificationModels}, cross-entropy–trained softmax models naturally provide a strong gradient signal for the ground-truth class when its logit is highly negative. In evidential models, however, the gradients produced by standard evidential activations vanish as the evidence approaches zero. To encourage a learning behavior closer to that of cross-entropy models in these regions, we propose introducing a regularization term $\mathcal{L}^{\texttt{cor}}(\mathbf{x}, \mathbf{y})$ that satisfies
\begin{align*}
&\frac{\partial \mathcal{L}^{\texttt{cor}}(\mathbf{x}, \mathbf{y})}{\partial o_{\text{gt}}} = -1 \quad \text{when} \quad \sum_{k=1}^{K} e_{k} = 0, \\
&\frac{\partial \mathcal{L}^{\texttt{cor}}(\mathbf{x}, \mathbf{y})}{\partial o_{\text{gt}}} \rightarrow -1 \quad \text{as} \quad \sum_{k=1}^{K} e_{k}  \rightarrow 0.
\end{align*}}

\bl{Motivated by this analysis, we propose the following vacuity-guided regularization:}

\bl{\begin{align}
\label{eq:correct_regularization}
    \mathcal{L}^{\texttt{cor}}(\mathbf{x}, \mathbf{y}) = - \lambda_{\text{cor}} o_{\text{gt}},
\end{align}
where $\lambda_{\texttt{cor}} = \nu = \frac{K}{S}$ denotes the vacuity produced by the evidential model. The vacuity behavior is characterized as $ \lambda_{\texttt{cor}} = 1 \texttt{ as } \sum_{k=1}^{K} e_{k} = 0 \quad \& \quad \lambda_{\texttt{cor}} \rightarrow 1 \texttt{ as } \sum_{k=1}^{K} e_{k} \rightarrow 0$. This choice ensures that the regularization magnitude approaches $1$ as the total evidence $\sum_{k=1}^{K} e_k$ approaches zero. To allow the evidential losses to dominate learning for high-evidence samples, we introduce an evidence-dependent indicator function in the loss:
\begin{align}\label{eqn:propCorEvReg}
    \mathcal{L}^{\texttt{cor}}(\mathbf{x}, \mathbf{y}) &= 
    \begin{cases}
        -\lambda_{\text{cor}}o_{\text{gt}} \quad &\text{if } o_{\text{gt}} < 0, \\
        0 \quad &\text{otherwise}
    \end{cases} 
    \;=\; 
    - \mathcal{I}\, \lambda_{\text{cor}} o_{\text{gt}},
\end{align}
where $\mathcal{I} = \mathbbm{1}(o_{\text{gt}} < 0)$ disables the regularization once the model assigns sufficiently positive evidence to the ground-truth class. The term is active primarily in low-evidence regions and diminishes as the sample moves away from the zero-evidence region ,with the greatest magnitude achieved in the \textit{zero evidence region}. Thus, it has the effect of pushing the samples away from the zero evidence region. This key property is summarized in the following theorem.}

\begin{theorem}
\label{th:sovlingzeroevidenceIssue}
\bl{Correct evidence regularization provides a non-vanishing gradient signal for training samples mapped to zero-evidence regions. }
\end{theorem}
\begin{proof}
\bl{The regularization depends only on the logit $o_{\text{gt}}$ of the ground-truth class. Hence,
\[
\frac{\partial \mathcal{L}^{\texttt{cor}}(\mathbf{x},\mathbf{y})}{\partial o_k}\Big|_{k \neq \text{gt}} = 0,
\]
and non–ground-truth nodes receive no update.
Because the indicator $\mathcal{I} = \mathbbm{1}(o_{\text{gt}} < 0)$ activates the term in low-evidence regions, we focus on this case. For $y_{\text{gt}} = 1$, the regularization and its gradient become}
\bl{\begin{align}
    \mathcal{L}^{\texttt{cor}}(\mathbf{x}, \mathbf{y}) &= - \lambda_{\text{cor}} o_{\text{gt}}, \\
    \frac{\partial \mathcal{L}^{\texttt{cor}}(\mathbf{x},\mathbf{y}) }{\partial o_{{gt}}} &= - \lambda_{\text{cor}}.
\end{align}
The vacuity $\lambda_{\text{cor}} = \frac{K}{S}$ lies in $[0,1]$ and achieves its maximum of $1$ when the evidence is zero. Thus, samples mapped to zero-evidence regions receive a gradient of $-1$, provide a meaningful update signal that promotes increased evidence for the ground-truth class. As evidence grows, vacuity decreases, and the influence of the regularization diminishes, allowing the standard evidential losses to guide learning for high-evidence samples.
Hence, the proposed regularization ensures non-zero gradients for zero-evidence samples and restores gradient flow in regions where standard evidential losses alone produce vanishing updates.}
\end{proof}

\subsection{Generalized Regularized Evidential Models}
The correct evidence regularization term in Eq.~\ref{eq:correct_regularization} is expressed in terms of the logit. \bl{When \texttt{SoftPlus} or Exponential activations are used, the regularization can also be written directly in terms of the output evidence because these activations are invertible. This is not the case for \texttt{ReLU}, whose non-invertibility prevents an evidence-based formulation. Theorem~\ref{th:superirorityofExp} indicates that the Exponential activation provides strong gradients for samples near the zero-evidence region, but its output may grow rapidly for large positive logits, making optimization more difficult. To balance these behaviors
 we introduce a novel evidential activation function, referred to as  \textbf{S}hifted \textbf{E}xponential \textbf{L}inear \textbf{Un}it (\texttt{SELU}) that generalizes existing activation functions with some appealing properties:
\begin{align}\label{eqn:selumain}
    e_i = \texttt{SELU}(o_i) = 
    \begin{cases}
        o_i + 1 \quad \text{if} \quad o_i > 0 \\
         \exp(o_i) \quad \text{otherwise} 
    \end{cases}
\end{align}
The activation behaves similarly to the Exponential activation function for negative logits and has the largest gradient (compared to \texttt{SoftPlus} and \texttt{ReLU}) for samples close to the \textit{zero evidence region}. For positive logits, the activation behaves linearly, and the evidence value does not explode as the logit value increases. The output evidence and the corresponding gradient plots from different activation functions are visualized in Figure \ref{fig:activationEvLogandGradPlot}. 
We present evidence-based formulation of the correct evidence regularization for different activation functions in Table \ref{tab:corEvRegForm}.}

\begin{figure}[t!]
    \centering
    \begin{subfigure}[b]{0.23\textwidth}
      \centering
      \includegraphics[width=\linewidth]{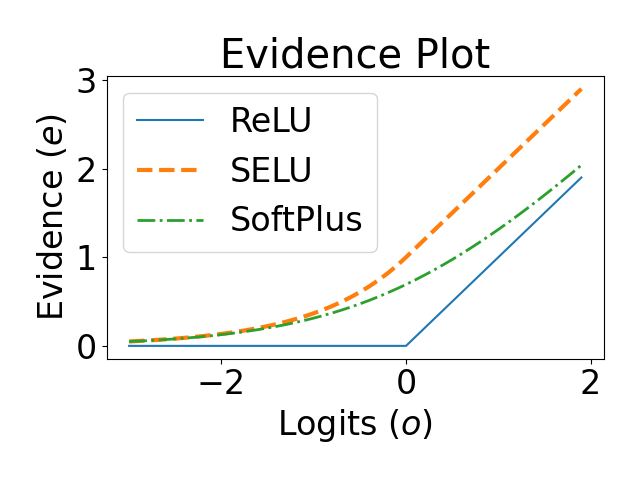} 
      \caption{Evidence-Logit Trend}
    \end{subfigure}
    \begin{subfigure}[b]{0.23\textwidth}
      \centering
      \includegraphics[width=\linewidth]{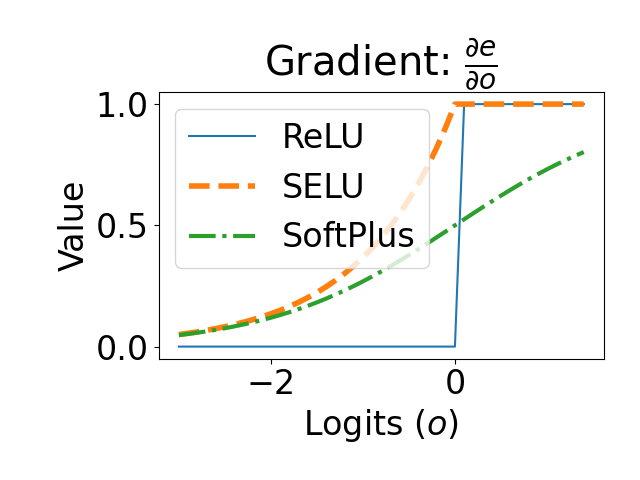} 
      \caption{Gradient Plot} 
    \end{subfigure}
\caption{Output Evidence and Gradient plot of different evidential activations for different Logit values} 
\label{fig:activationEvLogandGradPlot}
\end{figure}

\begin{table}[ht]
\centering
\small
\renewcommand{\arraystretch}{1.3} 
\setlength{\tabcolsep}{1pt} 
\caption{Evidence-Based Regularization}
\label{tab:corEvRegForm}
\begin{tabular}{cc>{\columncolor{black!10}}c} 
\hline 
\bf Activation & \bf Evidence & \bf Regularization ($\mathcal{L}^{\text{cor}}$) \\ 
\hline 
\texttt{ReLU}        & $e_i = \max(0, o_i)$                     & N/A \\ 
\hline 
\texttt{SoftPlus}    & $e_i = \log(\exp(o_i) + 1)$             & $-\mathcal{I}\lambda_{\text{cor}} \log(\exp(e_{\text{gt}}) - 1)$ \\ 
\hline 
\texttt{SELU}        & see Eq.~\ref{eqn:selumain}              & $-\mathcal{I}\lambda_{\text{cor}} \log(e_{\text{gt}})$ \\ 
\hline 
Exponential          & $e_i = \exp(o_i)$                       & $-\mathcal{I}\lambda_{\text{cor}} \log(e_{\text{gt}})$ \\ 
\hline
\end{tabular}
\end{table}

\subsection{Evidential Model Training}\label{sec:evModelTrainingLoss}

We formulate the overall objective used to train the proposed \textbf{G}eneralized \textbf{R}egularized \textbf{e}vidential mo\textbf{d}el (\textbf{RED}). The model is trained to increase evidence for the ground-truth class, reduce evidence for incorrect classes, and ensure that samples in low- or zero-evidence regions receive a meaningful learning signal. The \bl{combined} loss is
\begin{align}\label{eqn:proposedEvidentialModelOverallLoss}
    \mathcal{L}(\mathbf{x},\mathbf{y}) = \mathcal{L}^{\texttt{evid}}(\mathbf{x},\mathbf{y}) + \eta_1 \mathcal{L}^{\texttt{inc}}(\mathbf{x},\mathbf{y}) + \mathcal{L}^{\texttt{cor}}(\mathbf{x},\mathbf{y})
\end{align}
where $\mathcal{L}^{\texttt{evid}}(\mathbf{x},\mathbf{y})$ is the loss based on the evidential framework given by  Eq.~\ref{eqn:evMSEloss}, Eq.~\ref{eqn:evLogloss}, or Eq.~\ref{eqn:evDigammaloss} (See Appendix~\ref{apSec:evidentialProof}), $\mathcal{L}^{\texttt{inc}}(\mathbf{x},\mathbf{y})$ represents the incorrect evidence regularization (See Appendix Section \ref{app:evIncReg}), $ \mathcal{L}^{\texttt{cor}}(\mathbf{x},\mathbf{y})$ represents the proposed novel correct evidence regularization term in Eq.~\ref{eqn:propCorEvReg}, and $\eta_1 = \lambda_1 \times \min(1.0, \text{epoch index}/10) $ controls the impact of incorrect evidence regularization to the overall model training. In this work, we consider the forward-KL-based incorrect evidence regularization given in Eq.~\ref{appeq:klsensoy} based on \cite{sensoy2018evidential}.

\bl{\cyan{Figure \ref{fig:intuitiveFailureOfRelu} provides an intuitive view of learning in the evidence space. Ideally, samples from Class 1 should lie in the blue region (high evidence for Class 1), samples from Class 2 in the green region, and unseen or OOD samples in the zero-evidence region.} Training with $\mathcal{L}^{\texttt{evid}}$ and $\mathcal{L}^{\texttt{inc}}$ encourages these behaviors; however, when a sample is mapped to the zero-evidence region, the gradients of standard evidential losses vanish. Thus, although the true label is available, model does not update its knowledge when training such samples. Samples with low correct evidence and high incorrect evidence may also be driven toward this region (blue and green arrows in Figure~\ref{fig:intuitiveFailureOfRelu}), after which their update becomes inactive under standard evidential losses. This activation-dependent behavior occurs across evidential models.}

\bl{The GRED behavior is illustrated by the red arrows in Figure \ref{fig:intuitiveFailureOfRelu}. The correct evidence regularization is weighted by vacuity and therefore contributes most strongly in the zero-evidence region, where $\mathcal{L}^{\texttt{evid}}$ and $\mathcal{L}^{\texttt{inc}}$ provide no gradients. As evidence increases and vacuity decreases, the influence of $\mathcal{L}^{\texttt{cor}}$ fades, and the standard evidential losses dominate the learning signal. In this way, GRED ensures that samples across all evidence levels contribute to parameter updates while preserving the intended behavior of evidential training for high-evidence regions.}

\section{Experiments}

\bl{We evaluate our method across a broad range of benchmarks and architectures to validate the theoretical analysis, demonstrate the effectiveness of correct evidence regularization, and assess generalization and uncertainty quantification. Our experiments span standard supervised classification, few-shot learning, and a real-world image restoration task.}

\bl{We consider MNIST \cite{lecun1998mnist}, CIFAR-10 and CIFAR-100 \cite{krizhevsky2009learning}, and Tiny-ImageNet \cite{le2015tiny} for classification; $100$-way $1$-shot and $100$-way $5$-shot CIFAR-100 for few-shot learning; and \de{blind face restoration \cite{zhou2022towards} using FFHQ \cite{karras2019style} and CelebA \cite{karras2017progressive}}.}

\bl{To evaluate robustness across architectures, we employ LeNet\cite{lecun1999object} for MNIST, ResNet18 \cite{he2016deep} for CIFAR experiments, and Swin-Transformer \cite{liu2021swin, huynh2022vision} for Tiny-ImageNet. For few-shot learning, we use the transformer-based Visual Prompt Tuning (VPT) framework \cite{jia2022visual}, which adapts large pretrained vision transformers using lightweight prompts—\textit{a setting that stresses uncertainty estimation due to extremely limited supervision}. For blind face restoration, we use the VQGAN/Transformer-based CodeFormer model \cite{zhou2022towards}.}

\bl{We first present experiments that empirically verify the gradient behavior characterized in Section~III. We then evaluate the proposed correct evidence regularization on all datasets and architectures, followed by ablation studies analyzing the contribution of each evidential loss component and the impact on calibration and uncertainty metrics. We then extend the proposed evidential model to few-shot classification and blind face restoration, demonstrating consistent improvements across all settings. We then carry out out-of-distribution analysis of the proposed GRED model with challenging few-shot classification setting. Unless noted otherwise, table mean/std results are averaged over three seeds; training curves show one representative run.}
 Additional ablations, hyperparameter details, and clarifications are provided in the Appendix.

\bl{\subsection{Learning Dynamics and Failures in Evidential Models}}
\paragraph{Sensitivity to the change of the architecture.} 
We first consider a toy illustrative experiment with two frameworks: (1) a standard \texttt{Softmax} model and (2) an evidential model. \bl{Both use a LeNet \cite{lecun1999object} architecture similar to that considered in EDL~\cite{sensoy2018evidential}, with a minor modification to the architecture: dropout is removed. To construct the toy dataset, we randomly select 4 labeled data points from the MNIST training dataset, as shown in Figure \ref{fig:toyDatasetMnist}. For the evidential model, we use \texttt{ReLU} to transform the network outputs to evidence and train the model with the MSE-based evidential loss~\cite{sensoy2018evidential} given in Eq.~\ref{eqn:evMSEloss}, without incorrect evidence regularization. We train both models using only these 4 training data points. }

Figure \ref{fig:compToyMnistEVAccLossTrend} compares the training accuracy and loss trends of the evidential model with the standard softmax model (trained with cross-entropy\bl{). Before training, both models have $0\%$ accuracy and high loss, as expected. For the evidential model, in the first few iterations the accuracy increases to $50\%$, indicating that some samples are being fitted. Afterward, the accuracy plateaus: the evidential model maps two of the training samples to the \textit{zero evidence region}, where the gradients of standard evidential losses vanish. In this toy setting, the model therefore does not fully fit all four training points, empirically reflecting the behavior characterized in Theorem \ref{supotimalityTheorem}. It is also worth noting that the range of the evidential model's loss is significantly smaller than that of the standard model, mainly due to the bounded nature of the evidential MSE loss (\ie it lies in $[0,2]$; see the Appendix for a detailed theoretical analysis). By contrast, the standard model trained with cross-entropy easily fits the trivial dataset, reaching near-zero loss and $100\%$ accuracy after a few iterations. }
\begin{figure}[!t]
    \centering
    \includegraphics[width=.96\linewidth]{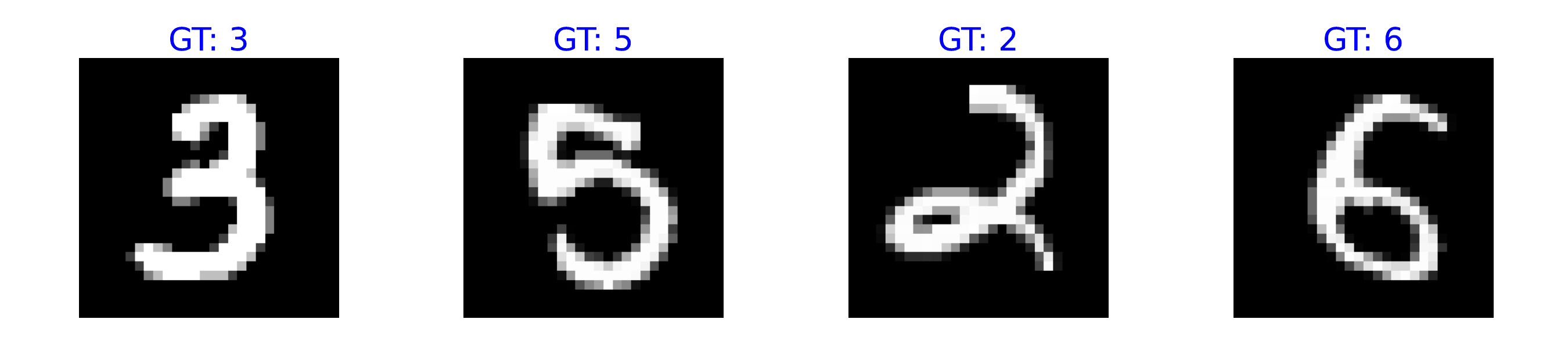}  
    \caption{Toy dataset with 4 data points } 
    \label{fig:toyDatasetMnist}
\end{figure}

\begin{figure}[t!]
    \centering
    \begin{subfigure}[b]{0.23\textwidth}
      \centering
      \includegraphics[width=\linewidth]{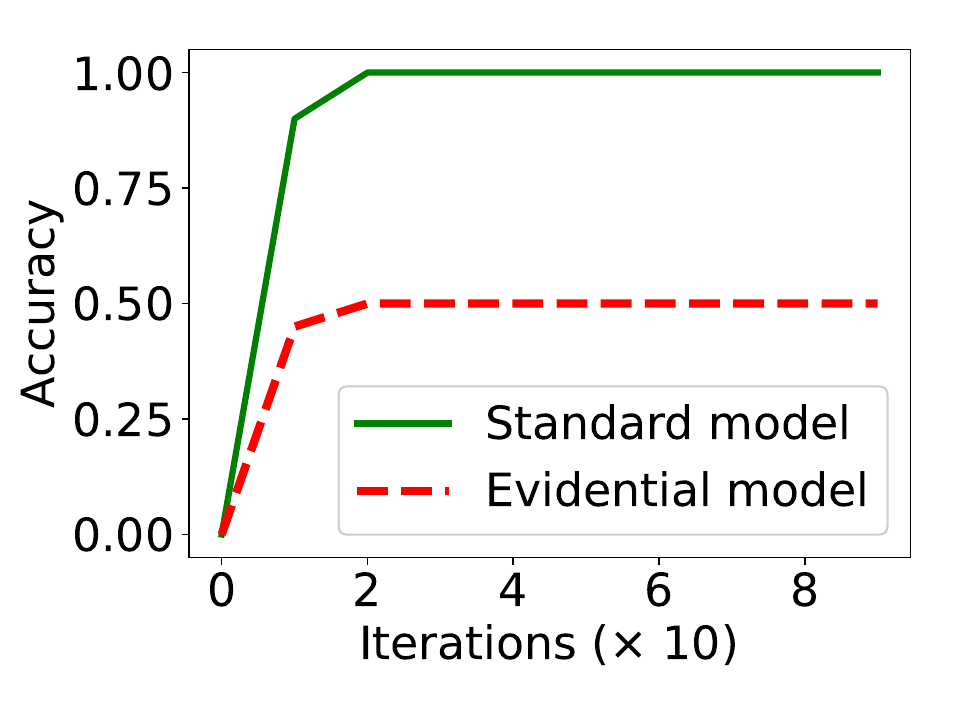} 
      \caption{Training Accuracy Trend}
    \end{subfigure}
    \begin{subfigure}[b]{0.23\textwidth}
      \centering
      \includegraphics[width=\linewidth]{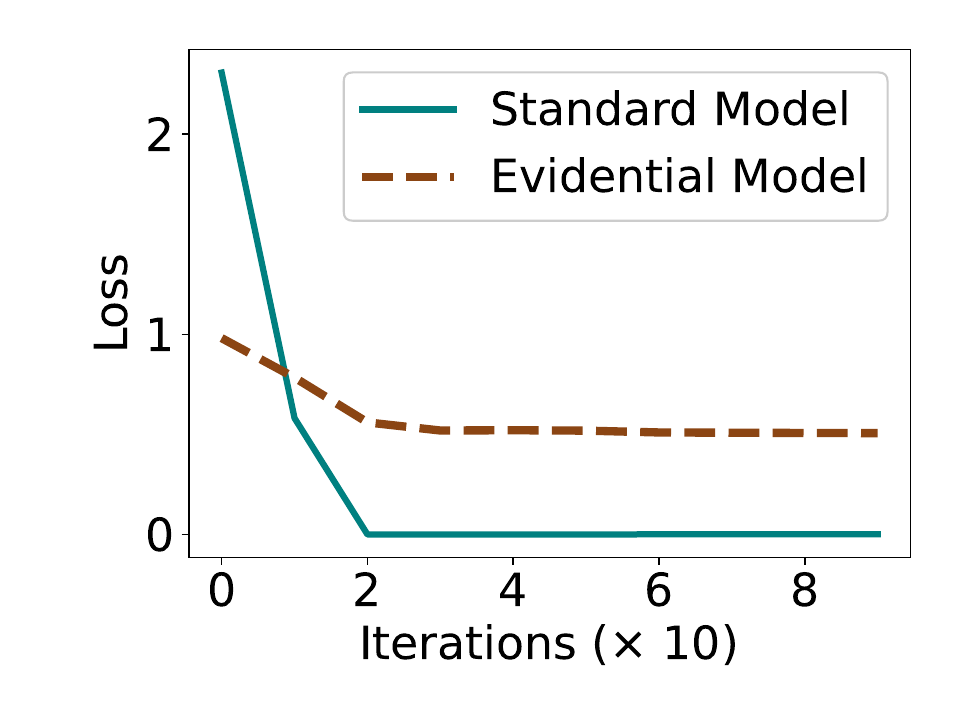} 
      \caption{Training Loss Trend} 
    \end{subfigure}
\caption{Training of standard and evidential models } 
\label{fig:compToyMnistEVAccLossTrend}
\end{figure} 

\begin{figure}[htpb]
    \centering
    \includegraphics[width=.7\linewidth]{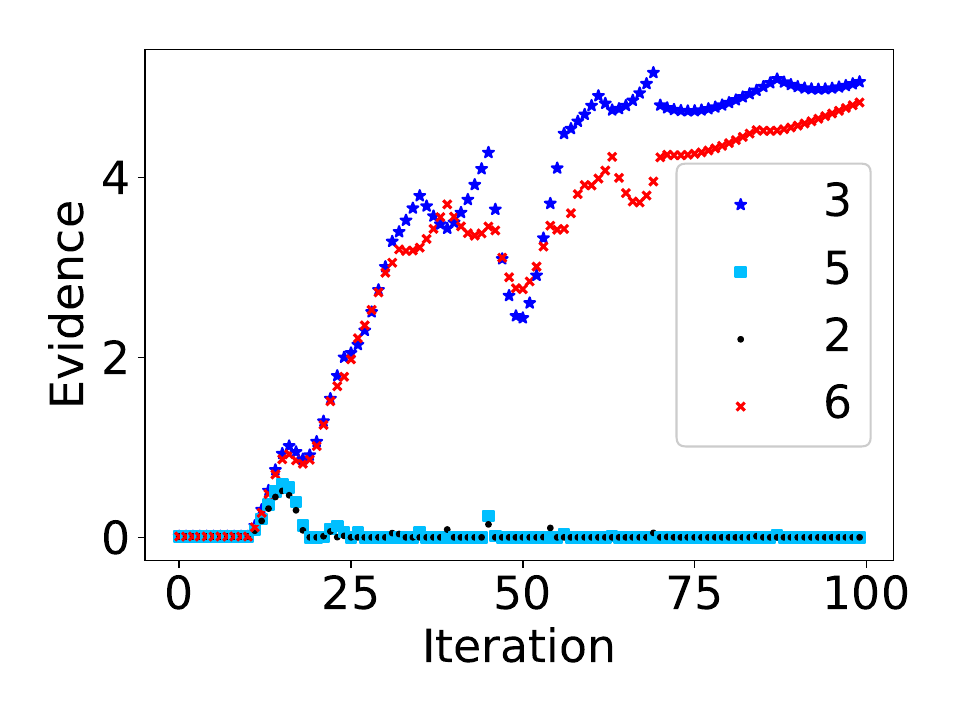}
    \caption{Zero evidence trend during model training} 
    \label{fig:zeroEvidenceTrendVisualization}
\end{figure}

\bl{Additionally, we visualize the total evidence for each training sample in this toy experiment. We plot the total evidence across the first 100 iterations in Figure \ref{fig:zeroEvidenceTrendVisualization}. The evidential model’s predictions are correct for the samples with ground-truth labels 3 and 6, and incorrect for the remaining two. After a few iterations, the latter two samples are mapped to the zero-evidence region, and their total evidence remains near zero. In this regime, the model receives no gradient from these samples, and the overall training accuracy stabilizes at $50\%$ even after 100 iterations. In contrast, the standard model continues to update on all four samples and achieves $100\%$ accuracy. This toy setting highlights how vanishing gradients in the zero-evidence region can affect evidential learning dynamics, even in simple cases.}

\paragraph{Sensitivity to hyperparameter tuning}
Evidential models are trained using evidential losses given in Eq.~\ref{eqn:evMSEloss}, Eq.~\ref{eqn:evDigammaloss}, or Eq.~\ref{eqn:evLogloss}  with incorrect evidence regularization to guide the model for accurate uncertainty quantification. We study the impact of the incorrect evidence regularization strength (\ie hyperparameter $\lambda_1$) on the evidential model's performance using CIFAR-100 experiments. We consider the Type-II Maximum Likelihood loss in Eq.~\ref{eqn:evLogloss} with different $\lambda_1$ to control KL regularization.  As shown in Figure \ref{fig:IncorrectEvRegImpact}, when some regularization is introduced, the evidential model's test performance improves slightly. However, when large regularization is used, the model focuses strongly on minimizing the incorrect evidence, \de{pushing a large number of training data samples to region near the \textit{zero evidence region}}. As can be seen, the generalization performance of evidential models is highly sensitive to $\lambda_1$ values. A similar trend is seen across all the losses and settings (results on other loss functions and settings are presented in the Appendix). \bl{Such incorrect evidence regularization can cause the model to push many training samples into or close to the zero-evidence regions, thereby reducing the effective learning signal from those samples. At the same time, incorrect evidence regularization is essential to correct incorrect acquired evidence and improve uncertainty estimates. Therefore, choosing a reasonable regularization strength is important for achieving accurate uncertainty quantification, especially on challenging datasets and settings, which we present next.}

\begin{figure}[t!] 
\centering
  \includegraphics[width=0.8\linewidth]{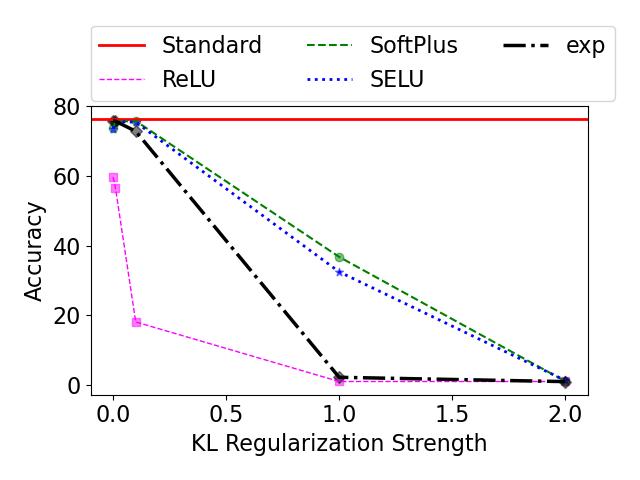}
\caption{Impact of different incorrect evidence regularization strengths to the test set accuracy on CIFAR-100}  
\label{fig:IncorrectEvRegImpact}
\end{figure} 

\paragraph{Challenging datasets and settings.}
We next consider a standard cross-entropy-trained classification model for the CIFAR-100 dataset and construct evidential extensions using the Type-II Maximum Likelihood loss in Eq.~\ref{eqn:evLogloss} and \de{the Bayes risk with cross-entropy loss in Eq.~\ref{eqn:evDigammaloss}} without incorrect evidence regularization, using \texttt{ReLU} to transform logits to evidence. \bl{As shown in Figure \ref{fig:ChallengingFailureEvidential}, compared to the standard classification model, the evidential models exhibit lower predictive performance (around $10\%$–$20\%$ lower for Eq.~\ref{eqn:evLogloss} and Eq.~\ref{eqn:evDigammaloss}, and more than $30\%$ lower for the MSE-based loss in Figure \ref{fig:cifar100MSEComp}). This behavior coincides with many training samples being mapped into or near the \textit{zero evidence region}, where the model expresses high vacuity and the gradients from standard evidential losses vanish. When incorrect evidence regularization is added, more samples can be driven toward the zero-evidence region, which may further reduce predictive accuracy if the regularization is too strong. In such cases, even though correct labels are available, the contribution of those samples to parameter updates becomes negligible once they reach near the zero-evidence region.}

\begin{figure}[t!] 
    \centering
    \begin{subfigure}[b]{0.46\linewidth}
      \centering
      \includegraphics[width=\linewidth]{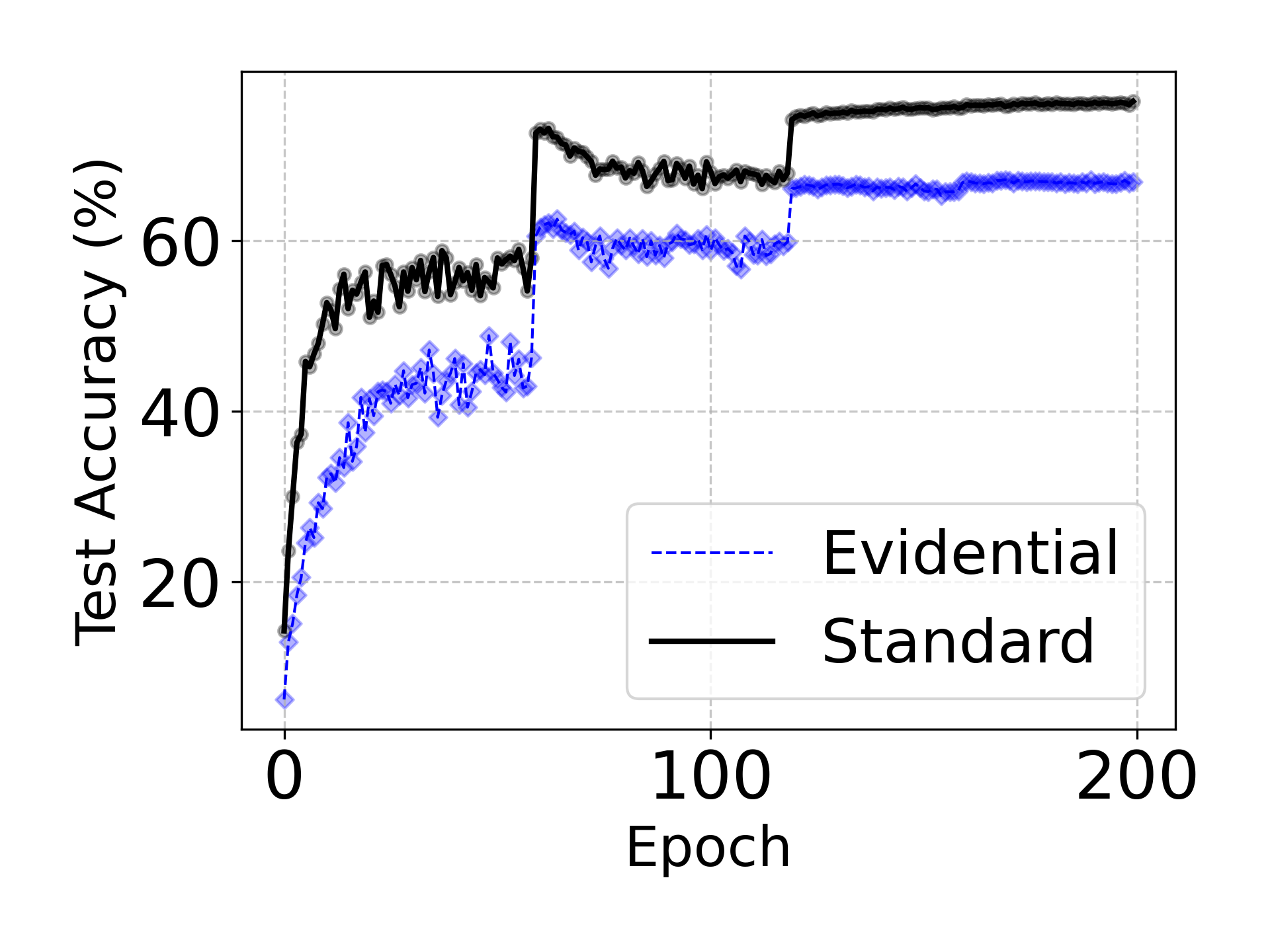}
      \caption{Evid. Log Loss in Eq.~\ref{eqn:evLogloss}}
    \end{subfigure}
    \begin{subfigure}[b]{0.46\linewidth}
    \centering
      \includegraphics[width=\linewidth]{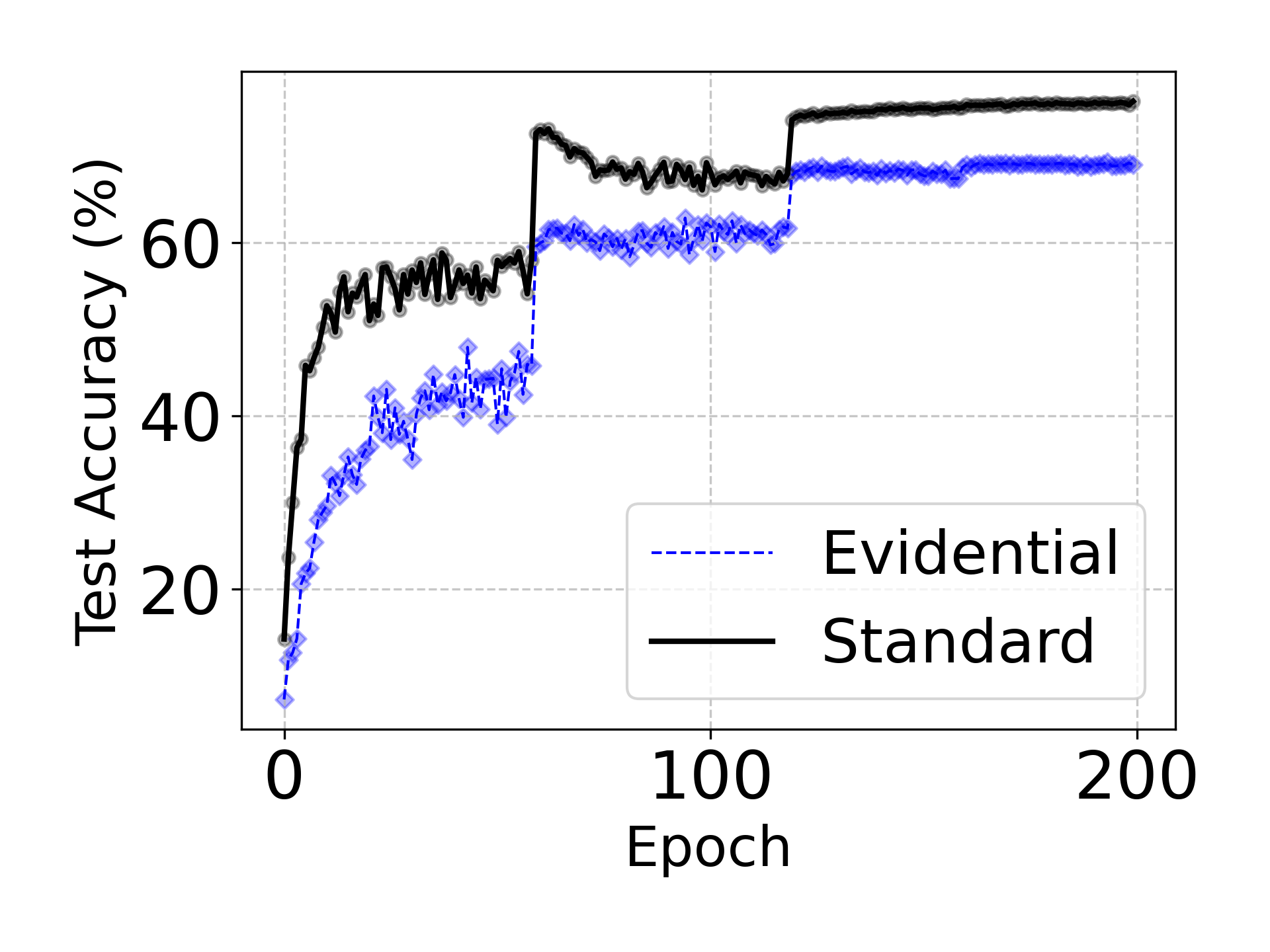}
      \caption{Evid. CE Loss in Eq.~\ref{eqn:evDigammaloss} }
      \label{fig:challengingDatasetsWeakness}
    \end{subfigure}
\caption{Learning trends in CIFAR-100 for standard and evidential models with different evidential losses}
\label{fig:ChallengingFailureEvidential}
\end{figure} 

\paragraph{ \bl{Sub-Optimal Learning} Caused by Incorrect Evidence Regularization}\label{sec:need_kl_edl}

\de{Existing evidential models are theoretically equipped to capture the fine-grained uncertainties through the higher-order conjugate prior distribution over the likelihood distribution. For classification, the evidential models introduce the Dirichlet prior over the multinomial likelihood distribution and train with evidential losses, such as Eq.~\ref{eqn:evLogloss}. Additionally, these evidential models leverage incorrect evidence regularization given in  Eq.~\ref{appeq:klsensoy} to ensure accurate uncertainty quantification, especially in the most challenging settings. 
To more clearly demonstrate the influence of the incorrect evidence regularization, we first consider the FGSM \cite{goodfellow2014explaining} adversarial attack applied to an evidential model with an Exponential activation function trained on CIFAR-100. We employ the evidential log loss, given in Eq.~\ref{eqn:evLogloss} and train the model for 200 epochs to study the accuracy-vacuity trends of the trained model for different strengths of the adversarial attack on the test set.} 

As shown in Fig \ref{fig:acc-vac-challenging-needkl-adv-attack}, as the attack strength increases, the overall accuracy of the model decreases. Since the evidential model is able to quantify the fine-grained uncertainty, we hope it can detect the attack through the predicted uncertainty. 
However, without the incorrect evidence regularization, the predicted vacuity remains low as the attack strength increases, as shown in Fig \ref{fig:acc-vac-challenging-needkl-adv-attack}(b).  
With a larger incorrect evidence regularization, the model becomes aware of its lack of knowledge for adversarial samples and outputs higher vacuity. However, when the incorrect evidence regularization is high, the model's learning capability becomes compromised: Fig \ref{fig:acc-vac-challenging-needkl-adv-attack}(a) shows that a larger $\lambda_1$ leads to a lower accuracy. Towards robust models, adversarial training methods have been developed and could be extended to evidential deep learning models \cite{kopetzki2021evaluating}. However, we observe that adversarial training of evidential models is sensitive to incorrect evidence regularization values (section \ref{subsed:gred_usefulness}, Figure \ref{fig:adv_train_results_evid}), and adversarial training becomes ineffective.
\begin{figure}[t!] 
    \centering
    \begin{subfigure}[b]{0.46\linewidth}
      \includegraphics[width=\linewidth]{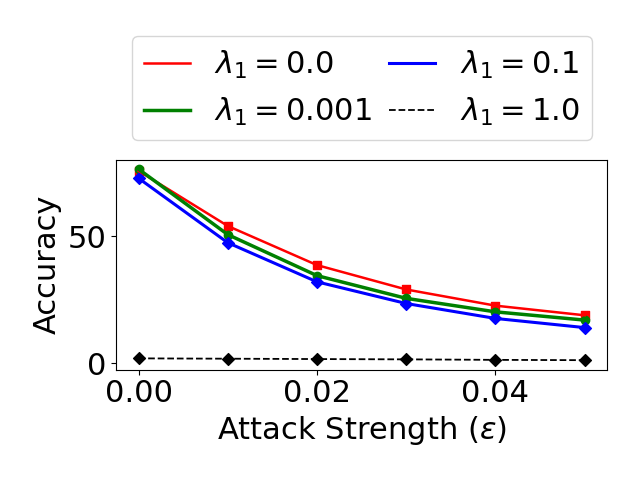}
      \caption{Accuracy Trend}
    \end{subfigure}
    \begin{subfigure}[b]{0.46\linewidth}
      \includegraphics[width=\linewidth]{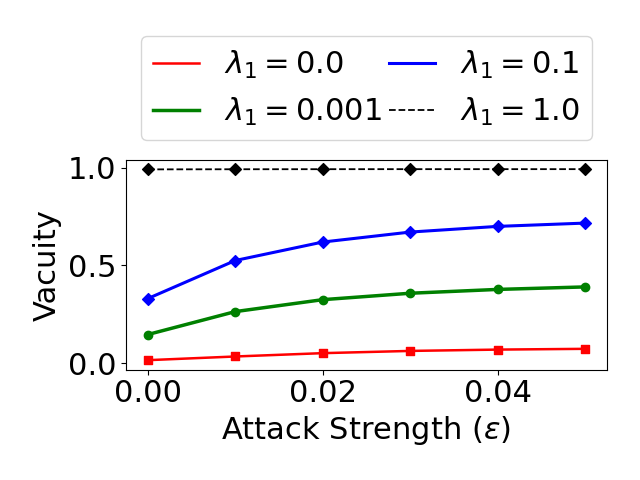}
    \caption{Vacuity Trend}
    \end{subfigure}
\caption{Adversarial attack trends for different incorrect evidence regularization strengths}
\label{fig:acc-vac-challenging-needkl-adv-attack}
\end{figure} 

To further illustrate the need for incorrect evidence regularization, we next present the accuracy-uncertainty results for the $1024-$class classification of the CodeFormer model for the FFHQ dataset. We consider the accuracy of the evidential transformer in the codebook prediction (details of the model are presented in the Appendix \ref{sec:facerestorationdetailsAppendix}) We present the accuracy-vacuity curves for the evidential models trained with and without incorrect evidence regularization term in Figure \ref{fig:acc-vac-challenging-needkl}. For a model with accurate uncertainty information, model's accuracy should be higher on low vacuity predictions. In other words, the model should be accurate on its confident predictions. However, when no incorrect evidence regularization is used, the model is wrongly confident on all code predictions i.e., the uncertainty is not reliable. Moreover, the vacuity is not expressive and is bound on a narrow range of $0.005$ to $0.0015$. With a reasonable incorrect evidence regularization value, \eg $\lambda_1 = 0.01$ for the model training, the accuracy-vacuity curves become more reasonable.  With a larger incorrect evidence regularization strength, model's accuracy in codebook prediction increases with lower vacuity threshold: model is accurate on the most confident predictions. 
However, the incorrect evidence regularization tends to push training samples towards the \textit{zero evidence region}, which hurts the model's training data efficiency, and the generalization capability. 

\begin{figure}[ht!] 
    \centering
    \begin{subfigure}[b]{0.46\linewidth}
      \includegraphics[width=\linewidth]{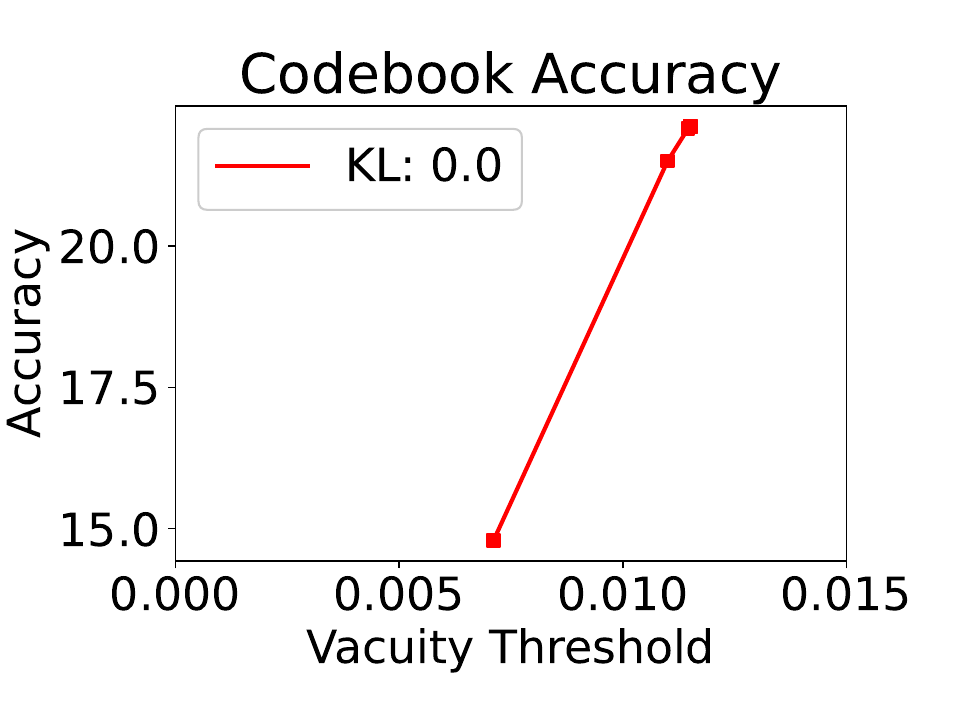}
      \caption{$\lambda_1 = 0$}
    \end{subfigure}
    \begin{subfigure}[b]{0.46\linewidth}
      \includegraphics[width=\linewidth]{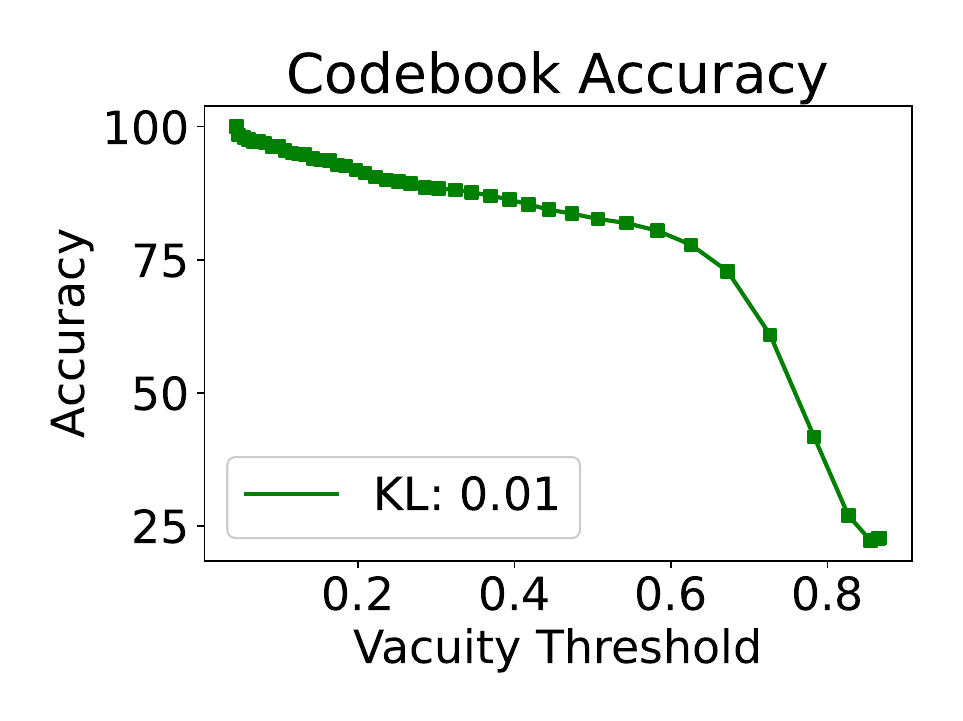}
    \caption{$\lambda_1 = 0.01$}
    \end{subfigure}
\caption{(a) Without Incorrect Evidence Regularization (b) With Incorrect Evidence Regularization of $\lambda_1$=0.01}
\label{fig:acc-vac-challenging-needkl}
\end{figure} 

\subsection{Generalized Regularized Evidential Models (GRED)}
\label{subsed:gred_usefulness}

\bl{We now evaluate the proposed generalized regularized evidential models, which enable evidential networks to learn from all samples, including those mapped to the \textit{zero evidence region}.} We experiment with multiple activation functions and their regularized variants (i.e., trained with the correct evidence regularizer) using the Type-II evidential loss in Eq.~\ref{eqn:evLogloss}. 

\bl{Across all datasets and architectures, introducing the correct evidence regularization consistently improves generalization (Table~\ref{tab:overallResultsStandard}), validating its effectiveness. Figure~\ref{fig:TrendCifar100IncRegMain} further shows that GRED remains stable even under strong incorrect evidence regularization, whereas baseline evidential models degrade because they cannot update on zero-evidence samples. Complete results and hyperparameter details are provided in the Appendix.}

\begin{figure}
\centering
  \includegraphics[width=0.9\linewidth]{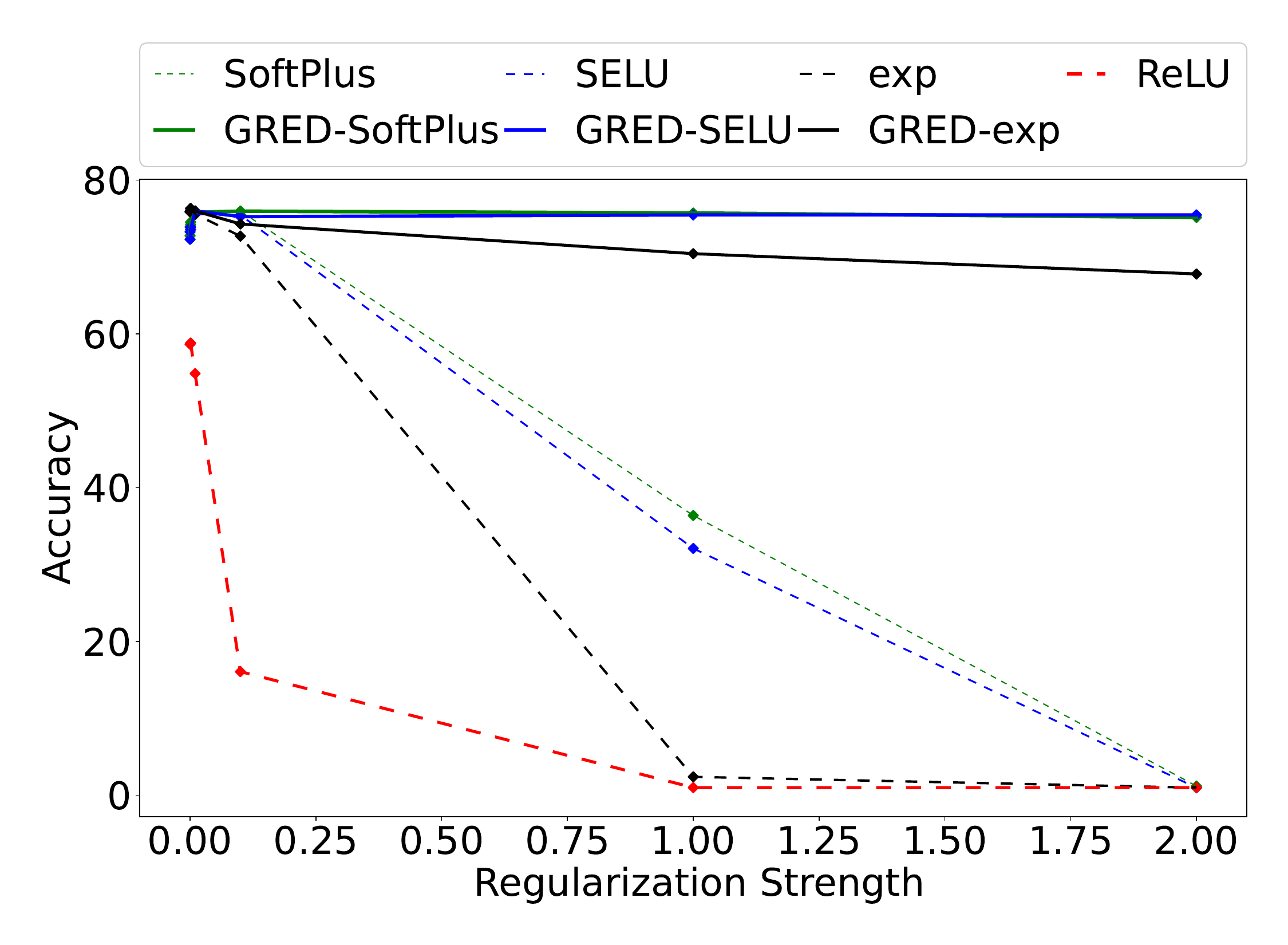}
\caption{Results for different regularization strengths for model trained with evidential log loss in Eq.~\ref{eqn:evLogloss} on CIFAR-100}
\label{fig:TrendCifar100IncRegMain}
\end{figure}

\begin{table*}[ht]
    \centering
    \caption{Evidential Classification Results. Mean and standard deviation are reported for each model after averaging across 3 runs. Bold values indicate the best performance for each dataset.}
    \label{tab:overallResultsStandard}
    \renewcommand{\arraystretch}{1.2}
    \setlength{\tabcolsep}{8pt}
    \begin{tabular}{lcc>{\columncolor{black!10}}cc>{\columncolor{black!10}}cc>{\columncolor{black!10}}c}
    \toprule
    $\bf{Dataset}$ & \texttt{ReLU} & \texttt{SoftPlus} & \bf{GRED}-\texttt{SoftPlus} & \texttt{SELU} & \bf{GRED}-\texttt{SELU} & $\exp$ & \bf{GRED}-$\exp$ \\
    \midrule
    \textbf{MNIST}          & $98.10_{\pm0.01}$  & $98.20_{\pm0.07}$  & $98.61_{\pm0.07}$  & $98.08_{\pm0.16}$  & $98.68_{\pm0.12}$  & $98.83_{\pm0.06}$  & \textbf{98.97}$_{\pm0.08}$ \\
    \textbf{CIFAR-10}       & $19.50_{\pm13.44}$ & $95.15_{\pm0.22}$  & $95.28_{\pm0.10}$  & $95.21_{\pm0.06}$  & $95.32_{\pm0.02}$  & $95.24_{\pm0.12}$  & \textbf{95.44}$_{\pm0.26}$ \\
    \textbf{CIFAR-100}      & $55.50_{\pm3.90}$  & $75.61_{\pm0.16}$  & \textbf{76.09}$_{\pm0.25}$  & $75.50_{\pm0.30}$  & $75.86_{\pm0.12}$  & $75.76_{\pm0.40}$  &  76.00$_{\pm0.09}$ \\
    \textbf{Tiny-ImageNet}  & $33.08_{\pm0.94}$  & $90.25_{\pm0.02}$  & \textbf{90.63}$_{\pm0.06}$  & $90.15_{\pm0.04}$  & $90.58_{\pm0.03}$  & $90.01_{\pm0.06}$  & \textbf{90.63}$_{\pm0.05}$ \\
    \bottomrule
    \end{tabular}
\end{table*}

\bl{We next examine training dynamics on MNIST with two evidential losses (Fig.~\ref{fig:new_Evid_actEvModel}). The $\exp$ activation performs strongest due to its minimal zero-evidence region, and adding correct evidence regularization further improves learning by ensuring that all samples contribute gradients.}

\begin{figure}[ht!] 
    \centering
    \begin{subfigure}[b]{0.46\linewidth}
      \includegraphics[width=\linewidth]{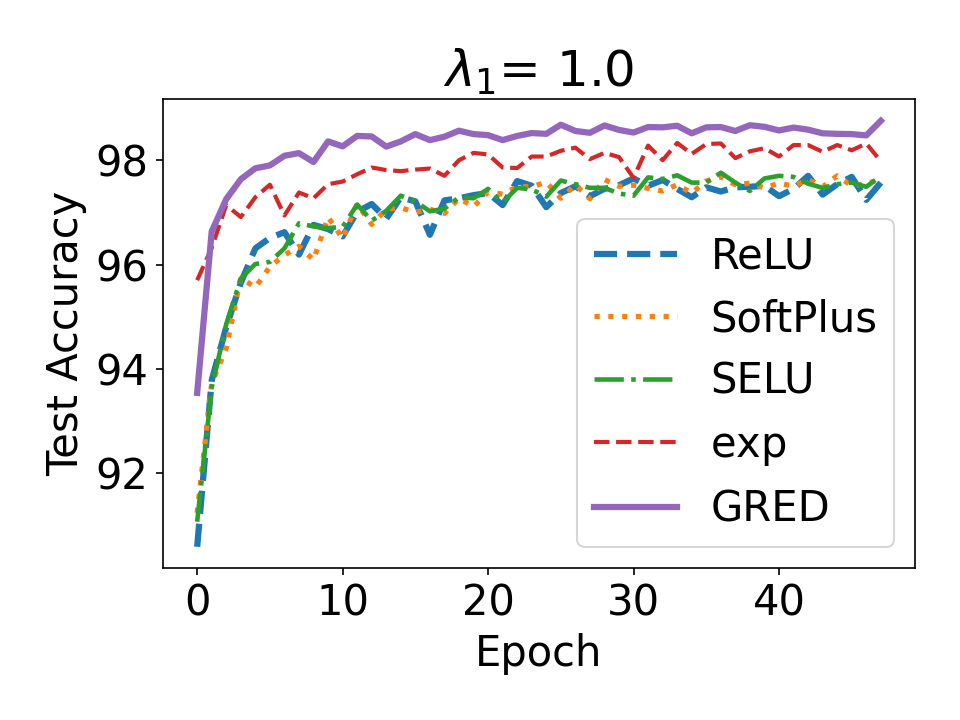}
      \caption{Evidential MSE loss}
    \end{subfigure}
    \begin{subfigure}[b]{0.46\linewidth}
      \includegraphics[width=\linewidth]{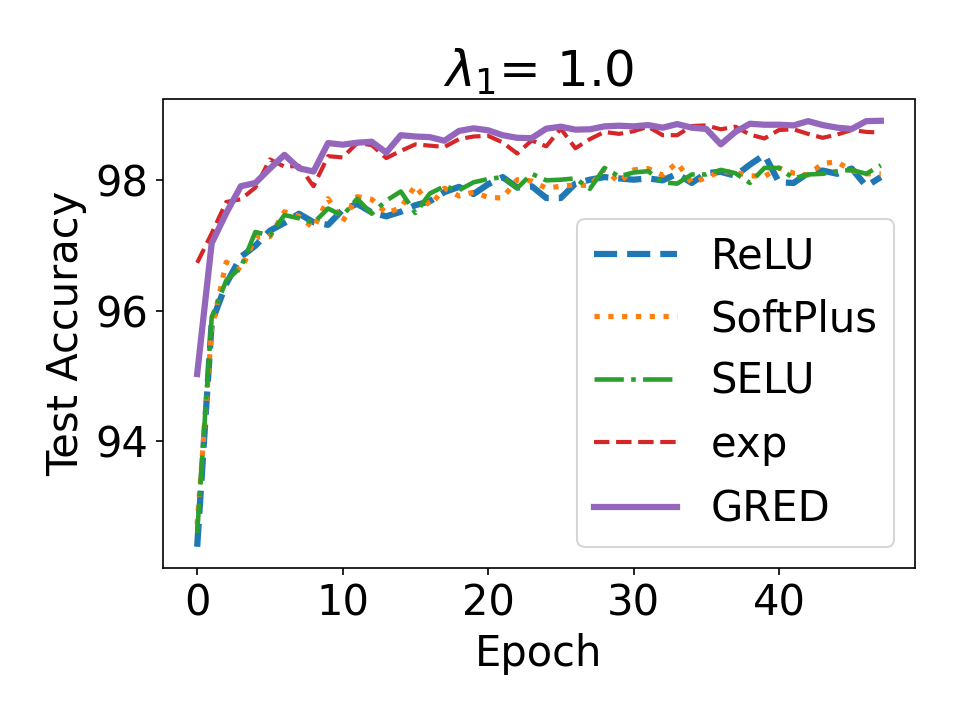}
    \caption{Evidential Log loss}
    \end{subfigure}
    
\caption{Test accuracy with correct evidence regularization}
\label{fig:new_Evid_actEvModel}
\end{figure} 

We further consider the evidential baseline model trained with Type-II Maximum Likelihood-based loss with incorrect evidence regularization strength of $\lambda_1 = 1.0$ and $10.0$. We introduce the proposed novel correct evidence regularization to the evidential model. As can be seen in Figure \ref{fig:CorEvRegImpact}, the model with correct-evidence regularization has superior generalization performance compared to the baseline evidential models in all the cases. This is mainly due to the fact that with proposed correct evidence regularization, the evidential model can also learn from the zero evidence training samples to acquire new knowledge instead of ignoring it. Our proposed model considers knowledge from all the training data and aims to acquire new knowledge to improve its generalization instead of ignoring the samples for which it has no knowledge. Finally, as seen in Figure \ref{fig:CorEvRegImpact} (d), even though strong incorrect evidence regularization hurts the model's generalization, the proposed model with correct evidence regularization is robust and generalizes better, empirically validating our Theorem \ref{th:sovlingzeroevidenceIssue}. Limited by space, we present additional results in Appendix \ref{app:secImpactOfCorrectEvReg} (additional results of CIFAR-100 are presented in the Appendix). 
\begin{figure}[ht!] 
\centering
    \begin{subfigure}[b]{0.46\linewidth}
      \centering
      \includegraphics[width=\linewidth]{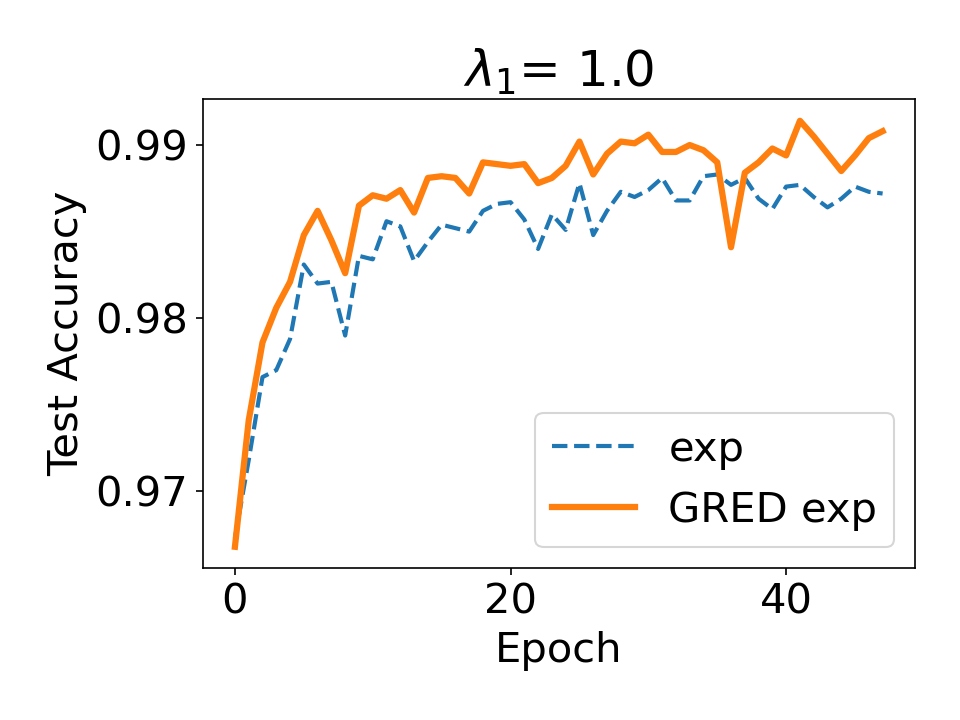}
      \caption{$\exp$-based Model}
    \end{subfigure}
    \begin{subfigure}[b]{0.46\linewidth}
      \centering
      \includegraphics[width=\linewidth]{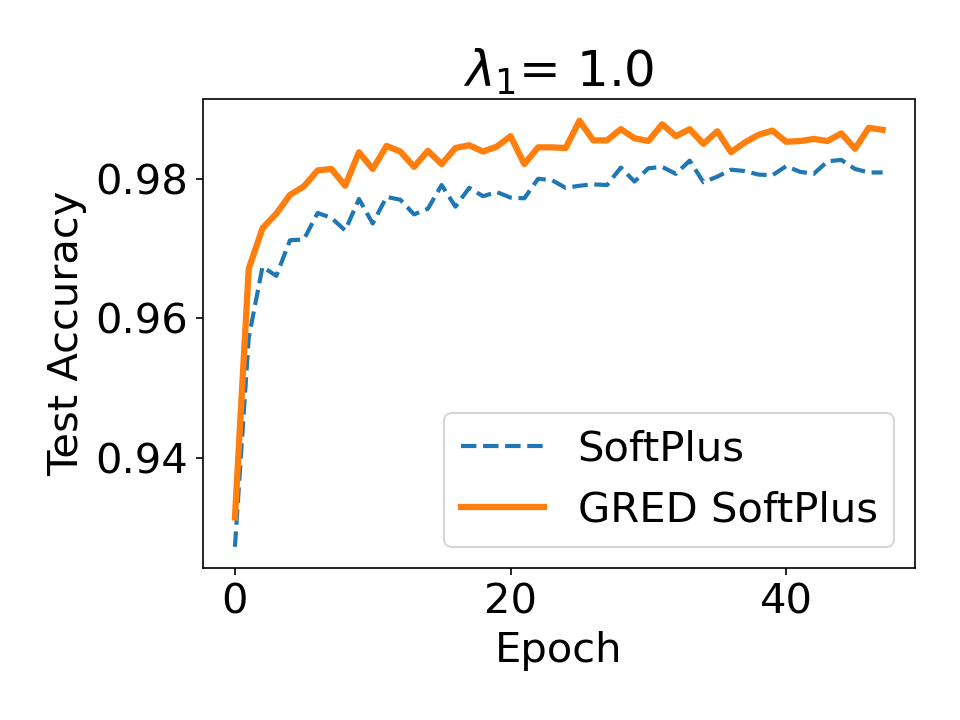}
      \caption{SoftPlus-based Model}
    \end{subfigure}
    \begin{subfigure}[b]{0.46\linewidth}
      \centering
      \includegraphics[width=\linewidth]{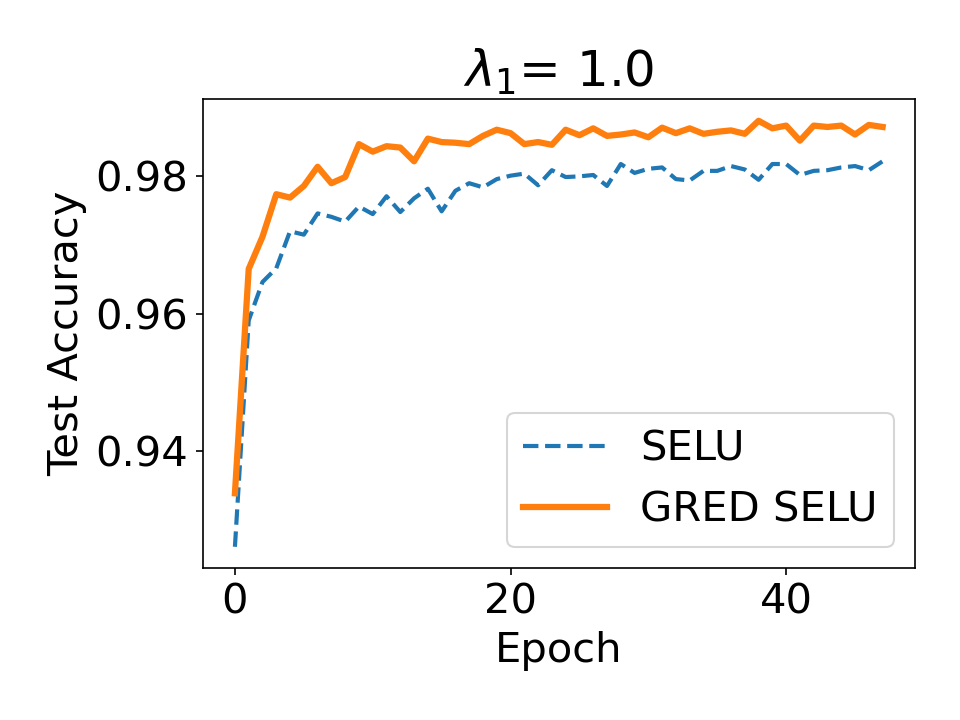}
      \caption{SELU-based Model}
    \end{subfigure}
    \begin{subfigure}[b]{0.23\textwidth}
      \centering
      \includegraphics[width=\linewidth]{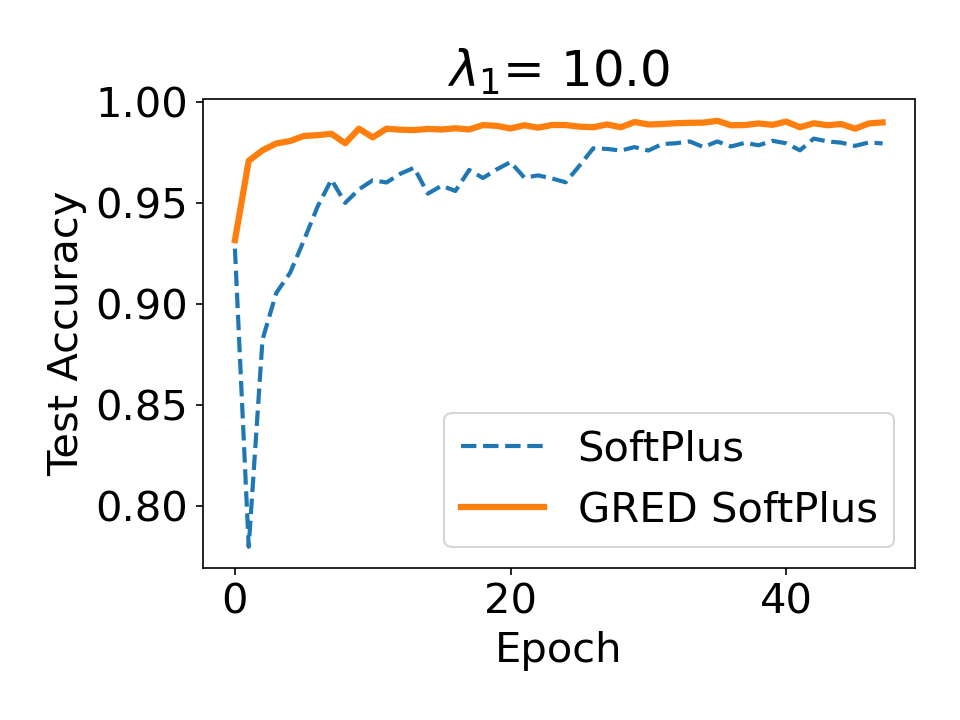}
      \caption{SoftPlus-based Model}
    \end{subfigure}
\caption{Impact of correct evidence regularization to test accuracy for different evidential models}
\label{fig:CorEvRegImpact}
\end{figure} 

\cyan{We evaluate the learning capabilities of evidential models under adversarial training. Specifically, we train an evidential model with an $\exp$ activation function using the evidential Type-II Maximum Likelihood-based loss, and incorrect evidence regularization strength of $\lambda_1 = 0.1$ and $1.0$ on the CIFAR-100 dataset. For generating adversarial samples, we apply the FGSM method with an attack strength of $\epsilon = 0.05$(additional details are provided in the Appendix \ref{apsubsec:hyperparameterdetails})}.\cyan{ The performance of the evidential models trained with incorrect evidence regularization values of $0.1$ and $1.0$ on the adversarial test set is presented in Fig. \ref{fig:adv_train_results_evid}. The evidential model struggles to learn from all training samples, 
resulting in poor performance on the adversarial test set. Meanwhile, using the proposed evidence regularizer enables the model to learn from all samples, leading to decent performance on the test dataset}
\begin{figure}[ht!] 
\centering
    \begin{subfigure}[b]{0.46\linewidth}
      \centering
      \includegraphics[width=\linewidth]{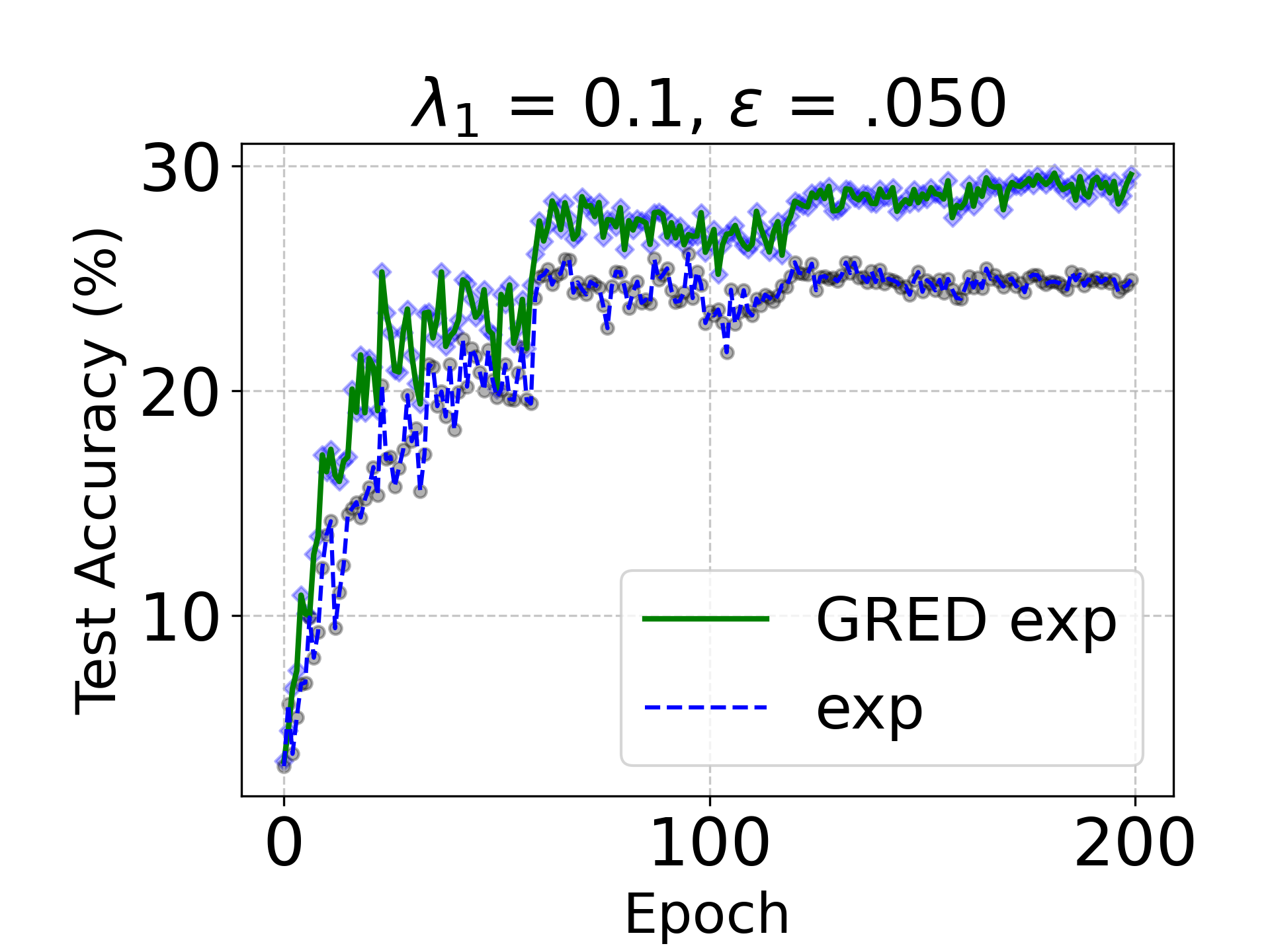}
      \caption{$\lambda = 0.1$}
    \end{subfigure}
    \begin{subfigure}[b]{0.23\textwidth}
      \centering
      \includegraphics[width=\linewidth]{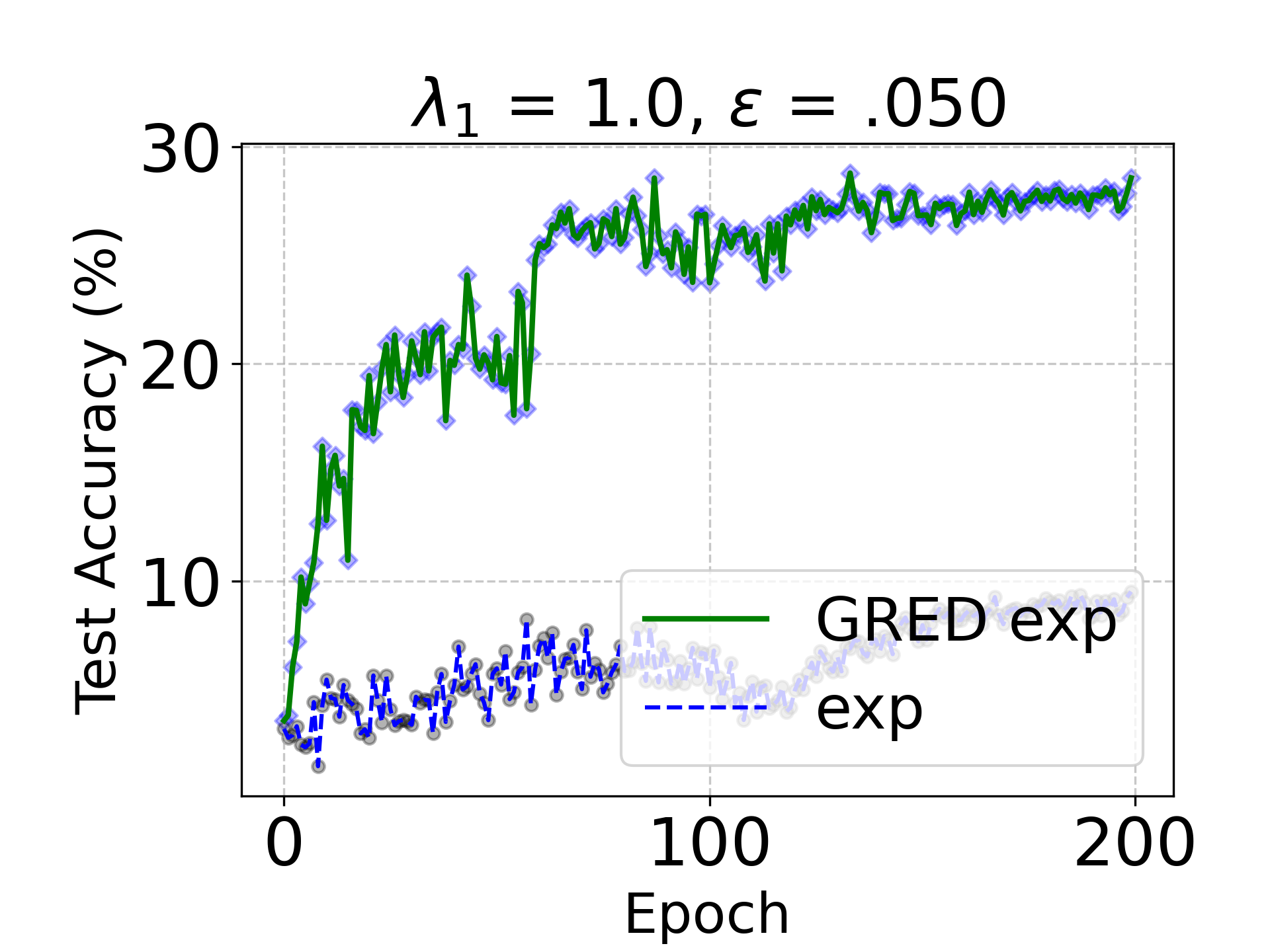}
      \caption{$\lambda = 1.0$}
    \end{subfigure}
\caption{Impact of proposed regularization $\mathcal{L_{\texttt{cor}}}$ to Adversarial Training of evidential models}
\label{fig:adv_train_results_evid}
\end{figure}

\subsection{Ablation Study}
\paragraph{Impact of loss function.}
We \bl{first} examine the effect of different evidential losses on MNIST using the $\exp$ activation and incorrect-evidence regularization strengths $\lambda_1 \in \{0,1\}$. \bl{Training with the evidential MSE loss (Eq.~\ref{eqn:evMSEloss}) consistently yields lower test performance than the other two losses (Eq.~\ref{eqn:evDigammaloss} and Eq.~\ref{eqn:evLogloss}). This behavior is expected because the evidential MSE loss is bounded in $[0,2]$, which restricts gradient magnitude and slows learning. Additional activation-wise comparisons and theoretical discussion are provided in the Appendix.}

\bl{Unless otherwise stated, subsequent experiments use the $\exp$ evidential activation with the Type-II Maximum Likelihood loss (Eq.~\ref{eqn:evLogloss}), which offers stable optimization and a clearer probabilistic interpretation. A deeper comparison between Eq.~\ref{eqn:evDigammaloss} and Eq.~\ref{eqn:evLogloss} is left for future work.}

\begin{figure}[ht!] 
\centering
\begin{subfigure}[b]{0.46\linewidth}
  \includegraphics[width=\linewidth]{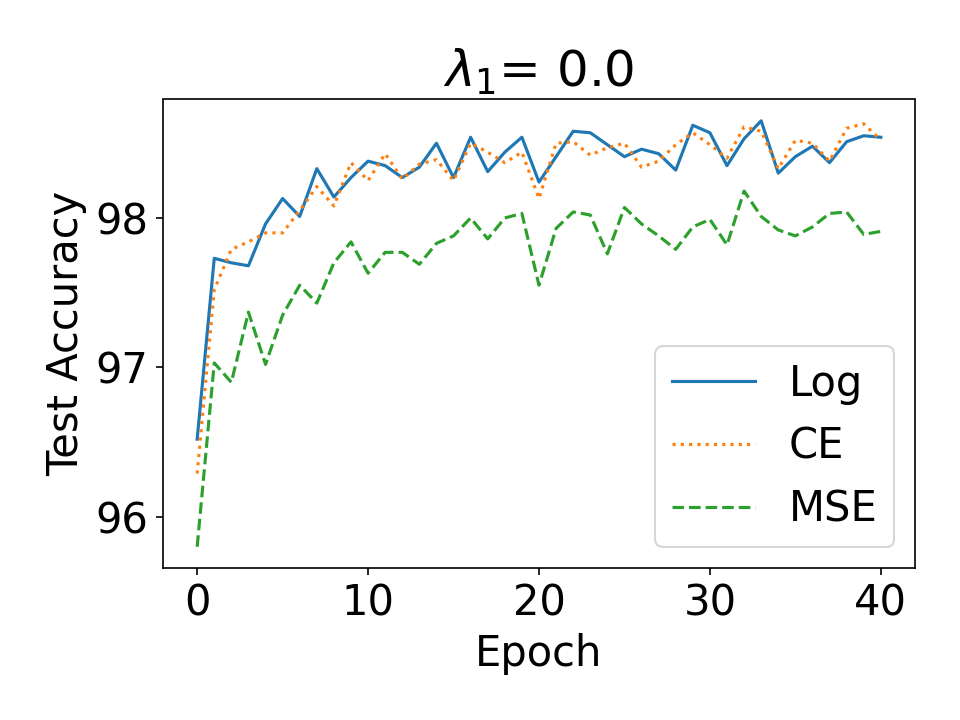}
  \caption{Trend for $\lambda_1 = 0.0$}
\end{subfigure}
\begin{subfigure}[b]{0.46\linewidth}
  \includegraphics[width=\linewidth]{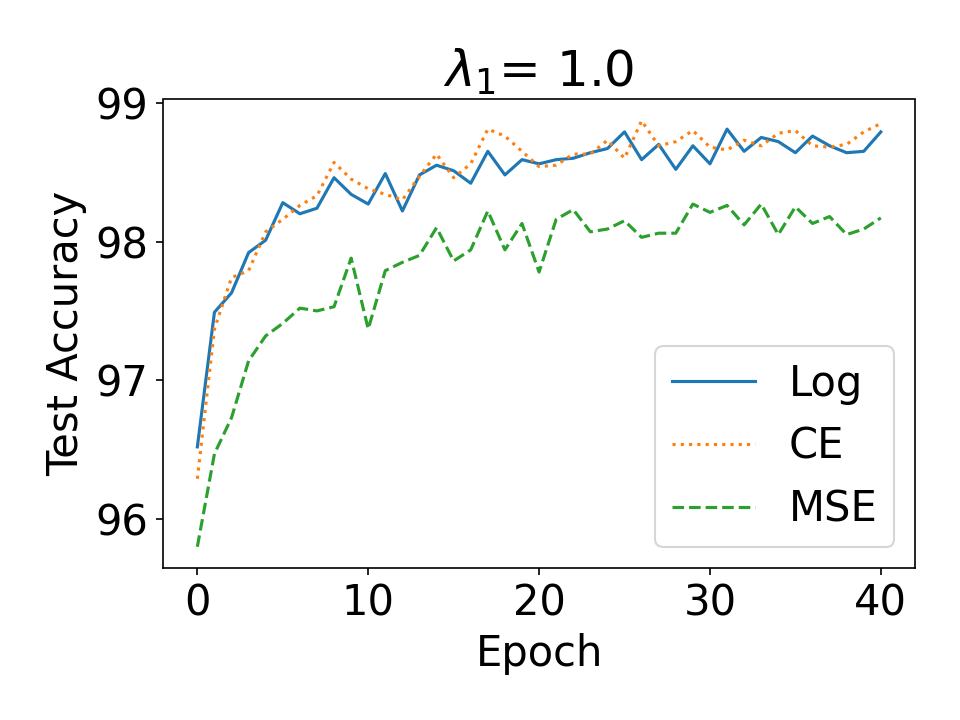}
\caption{Trend for $\lambda_1 = 1.0$}
\end{subfigure}
\caption{Impact of evidential losses on test set accuracy}
\label{fig:LossImpactEvidentialModel}
\end{figure}

\paragraph{\bl{Uncertainty vs Accuracy Trends}}

\bl{We next analyze the uncertainty behavior on CIFAR-100. Figure~\ref{fig:accVacuityCurve} shows accuracy–vacuity curves for different incorrect-evidence regularization strengths $\lambda_1$.} Vacuity reflects the lack of confidence in the predictions, and the accuracy of an effective evidential model should increase with a lower vacuity threshold. Without any incorrect evidence regularization (\ie $\lambda_1 = 0$), the evidential model is highly confident in its predictions and all test sample predictions are concentrated on the low vacuity region. As the incorrect evidence regularization strength is increased, the model outputs more accurate confidence in the predictions. Strong incorrect evidence regularization hurts the generalization over the test set as indicated by low accuracy when all test samples are considered (\ie vacuity threshold of 1.0). Our correct-evidence regularized evidential model shows reasonable uncertainty behavior: the model's test set accuracy increases as the vacuity threshold is decreased. 
 \begin{figure}[ht!] 
\centering
  \includegraphics[width=0.7\linewidth]{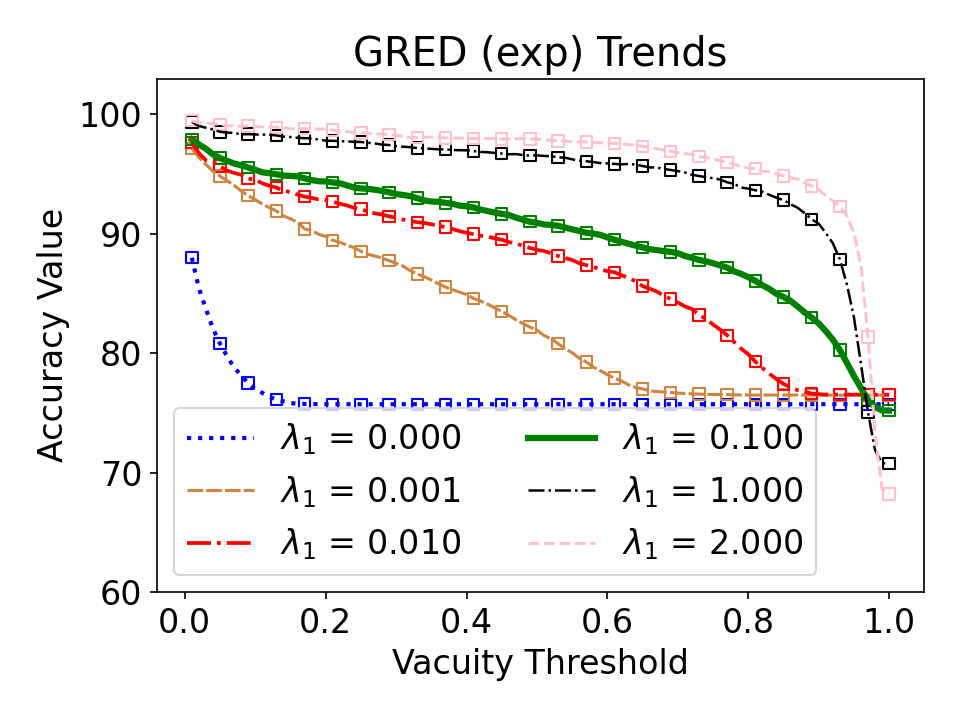}
\caption{Accuracy-Vacuity curve }
\label{fig:accVacuityCurve}
\vspace{-1mm}
\end{figure}

\bl{We further evaluate accuracy on the top-$K\%$ most confident test predictions (Table~\ref{tab:topKConfidentAccuracyTrend}). For example, among the top 20\% most confident samples, the proposed GRED model achieves 99.40\% accuracy, outperforming all baseline models. Across all values of $K$, GRED matches or exceeds the strongest baseline, demonstrating that the improved uncertainty estimates translate into better ranking of prediction confidence.}

\begin{table}[ht]
\centering
\small
\setlength{\tabcolsep}{5pt} 
\renewcommand{\arraystretch}{1.2} 
\caption{Accuracy on Top-K$\%$ Confident Samples (\%)}
\label{tab:topKConfidentAccuracyTrend}
\begin{tabular}{lcccccc}
\toprule
\bf{Model} & \bf 10\% & \bf 20\% & \bf 50\% & \bf 70\% & \bf 80\% & \bf 100\% \\
\midrule
\texttt{ReLU}         & 99.20  & 98.45  & 90.60  & 77.74  & 70.28  & 55.50  \\
\texttt{SELU}         & 99.30  & 99.00  & 96.02  & 89.87  & 85.84  & 75.50  \\
\texttt{SoftPlus}     & 98.70  & 98.65  & 95.94  & 89.91  & 85.80  & 75.61  \\
\texttt{Exp}          & 99.40  & 99.20  & \bf 96.54  & \bf 90.46  & 86.19  & 75.76  \\
\midrule
\rowcolor{black!10} \bf{GRED}-Exp  & \bf 99.60  & \bf 99.40  & 96.40  & \bf 90.46  & \bf 86.26  & \bf 76.00  \\
\bottomrule
\end{tabular}
\vspace{-3mm} 
\end{table}

\bl{\paragraph{Calibration Analysis.}
We evaluate calibration using Expected Calibration Error (ECE) on the CIFAR-100 test set. Figure~\ref{fig:calibration_analysis} reports ECE trends across KL regularization strengths and the reliability diagrams. Across most settings, GRED achieves calibration on par with vanilla EDL, with marginal improvements in some hyperparameter values. As KL regularization increases, both models become increasingly underconfident, reflected by rising ECE values as seen in Figure~\ref{fig:calibration_analysis}(b)-(d) reliability diagrams.}

\bl{At large KL values (e.g., $\lambda_1 = 1.0$), EDL \cite{sensoy2018evidential} exhibits a deceptively low ECE (1.6\%) due to model collapse: its accuracy drops to 34.3\%, and the model outputs near-uniform predictions (“I don't know”) for many test samples. In contrast, GRED maintains functional predictive performance (70.`\% accuracy, 26.4\% ECE). This highlights how ECE alone can be misleading when a model collapses to high-vacuity predictions. Addressing the broader underconfidence trend observed in evidential models is an important direction for future work.}

\bl{
\begin{figure}[ht!]
\centering
\begin{subfigure}[b]{0.46\linewidth}
  \includegraphics[width=\linewidth]{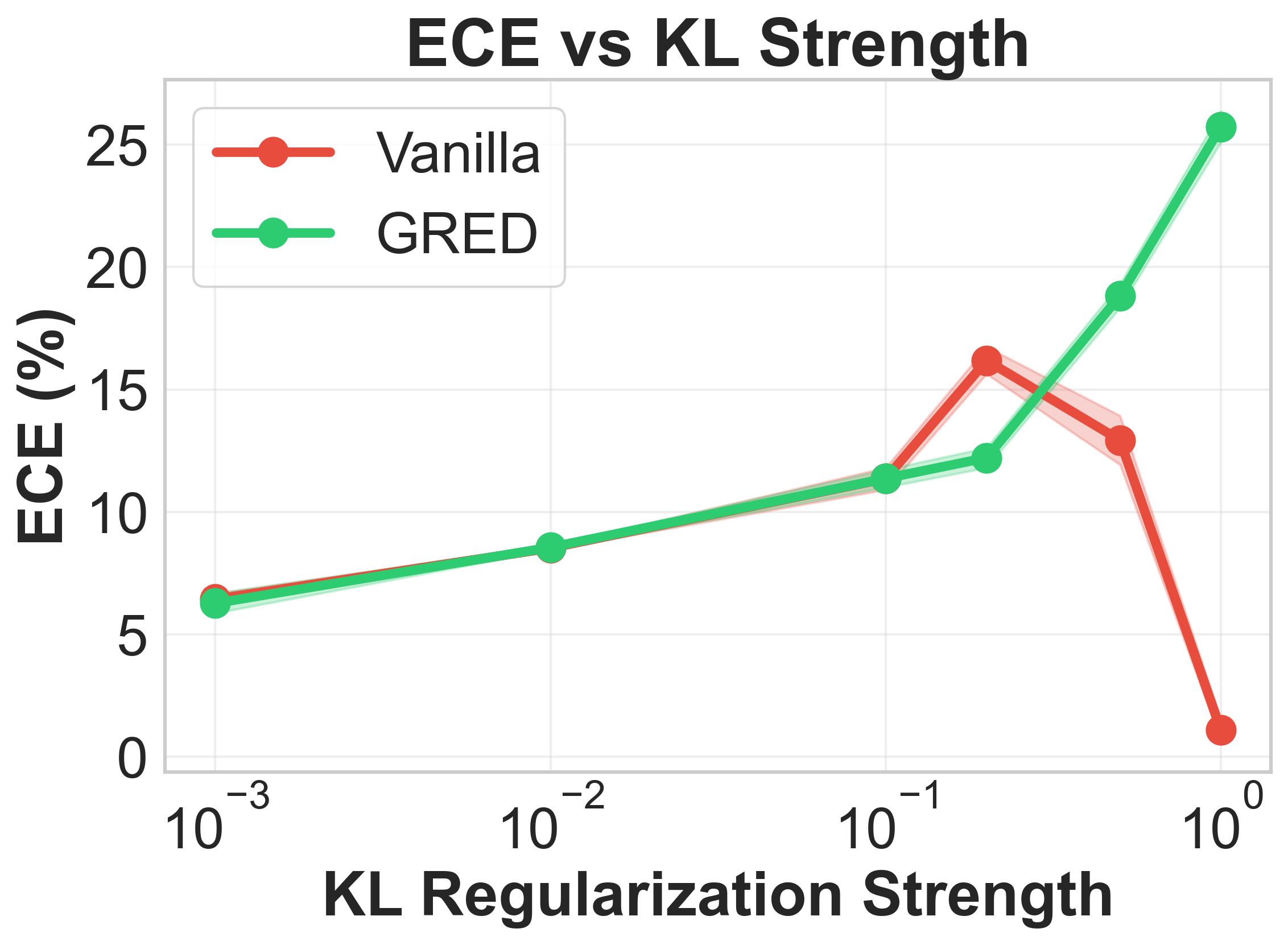}
  \caption{ECE vs KL strength}
\end{subfigure}
\begin{subfigure}[b]{0.46\linewidth}
  \includegraphics[width=\linewidth]{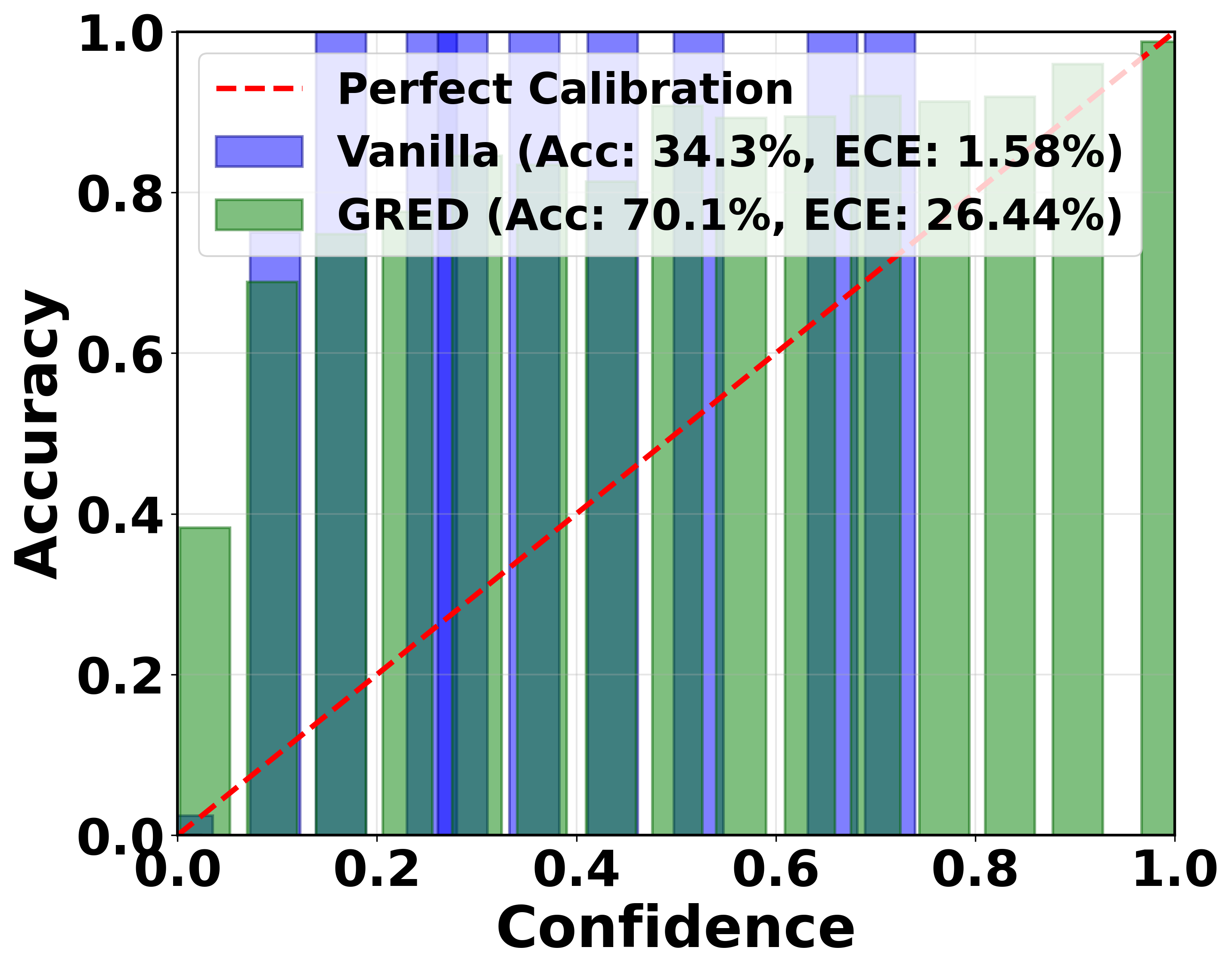}
  \caption{KL = 1.0 }
\end{subfigure}
\begin{subfigure}[b]{0.46\linewidth}
  \includegraphics[width=\linewidth]{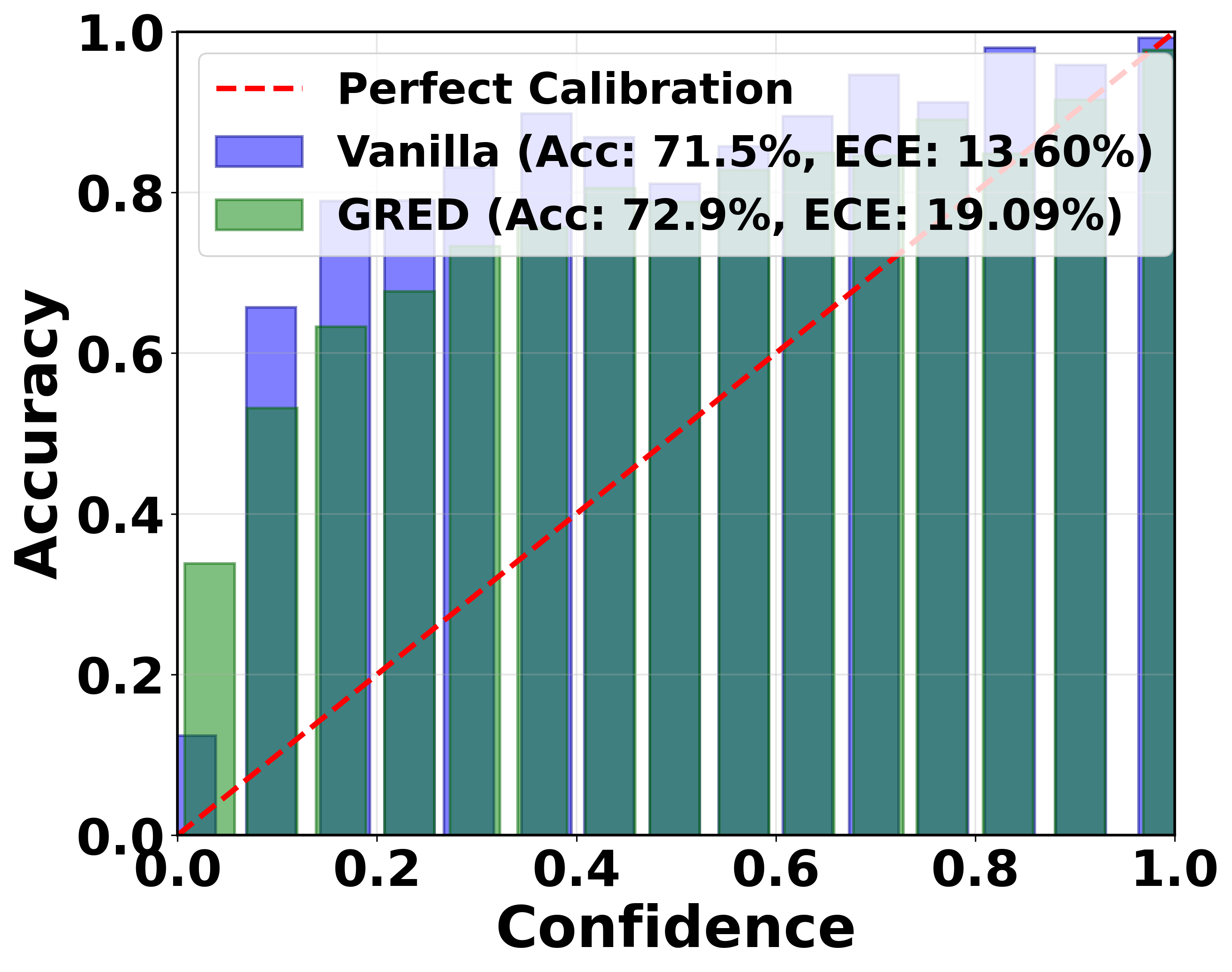}
  \caption{KL = 0.5 }
\end{subfigure}
\begin{subfigure}[b]{0.46\linewidth}
  \includegraphics[width=\linewidth]{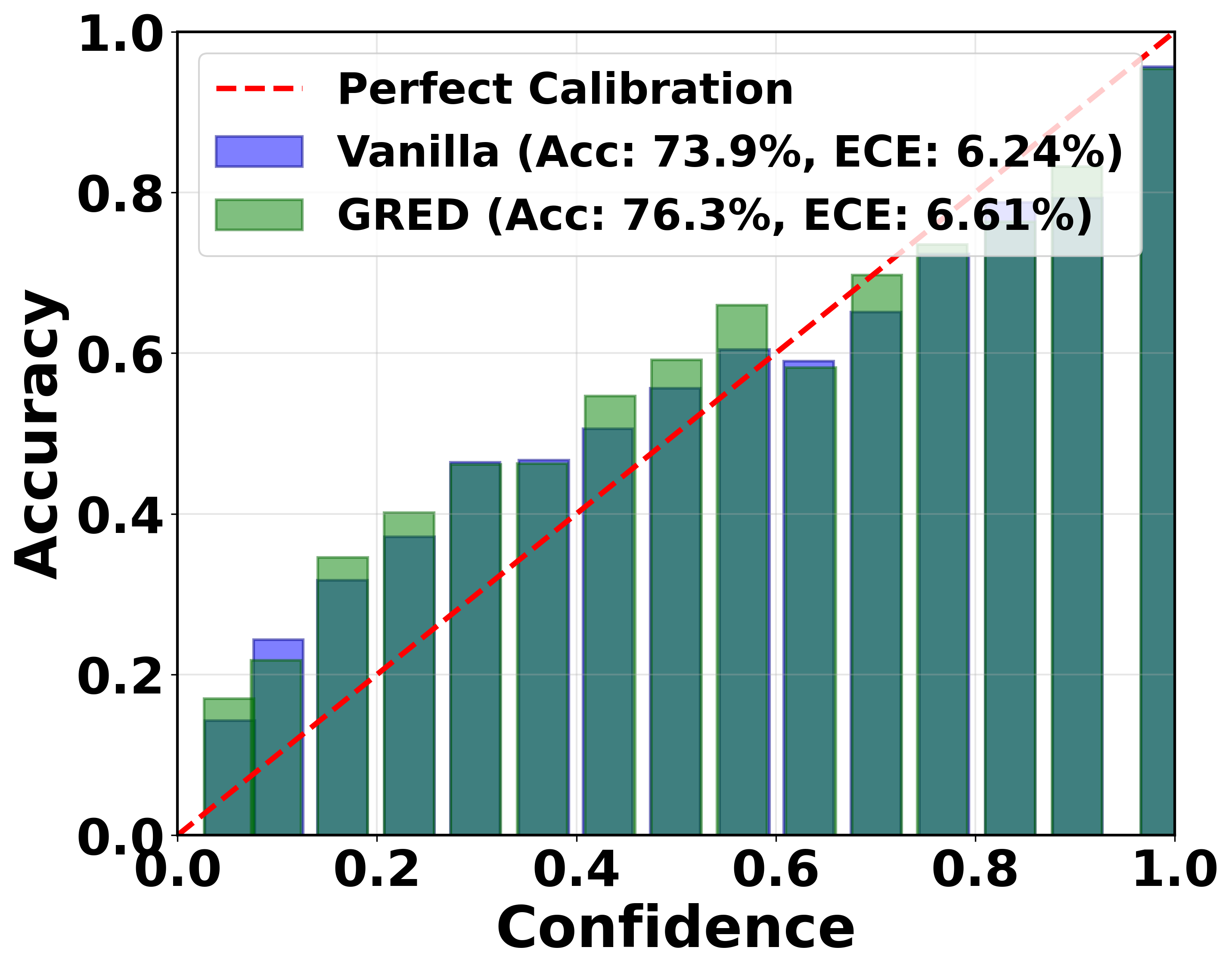}
  \caption{KL = 0.001 }
\end{subfigure}
\caption{Calibration analysis with log loss. (a) ECE trends (b)-(d) Reliability diagrams comparing GRED with EDL for different KL regularization values}
\label{fig:calibration_analysis}
\end{figure}
}

\subsection{Extension to Blind Face Restoration}
\de{Codeformer models \cite{zhou2022towards} have been developed that introduce VQ-GAN-based networks with a transformer architecture leading to state-of-the-art blind face restoration performance. However, the blind face restoration problem is ill-posed by
nature, and many of the restored faces are unlikely to be faithful to the true face
images. Moreover, the blind face restoration problem, by design, introduces uncertainty in the downstream restoration task. In this section, we extend the codeformer model using the ideas from our generalized regularized evidential deep learning model to develop evidential codeformers with fine-grained uncertainty-quantification capabilities. We use the fine-grained uncertainty information to improve the codebook lookup of the codeformer, which leads to significant improvement on the blind face restoration benchmark, demonstrating the potential of the uncertainty-aware model. }

\de{
The encoder-decoder based codeformer models \cite{zhou2022towards} are trained in 3 stages: 1) Stage I with high-quality images to learn the codebook and train the decoder parameters, 2) Stage II with low-quality and high-quality image pairs to train the transformer classifier along with the encoder structure, and 3) an optional Stage III training of controllable feature transformation module to find good balance between the quality and fidelity of face restoration. We modify the transformer introduced in Stage II to an evidential transformer structure by replacing the softmax layer in the transformer with \cyan{the evidential activation functions}. We train the evidential transformer on the FFHQ dataset based on Type II likelihood-based evidential loss in Eq.~\ref{eqn:evLogloss} with incorrect evidence regularization strength of $\lambda_1 = 1.0$. The evidential transformer outputs the $K$-dimensional evidence vector to identify $D$-dimensional code item from the $K \times D$ shaped codebook for each of the $M$ positions of the decoder input. For each position of the decoder input, based on the $K$-dimensional evidence vector $\mathbf{e} = (e_1, e_2, ..., e_{K})^\top$, the vacuity $\nu$, and $K$-dimensional belief vector $\mathbf{b} = (b_1, b_2, ..., b_{K})^\top$ can be computed. In all codeformer-based experiments, $K = 1024, D = 256,$ and the decoder input is a $16\times16\times256$ shaped tensor with $M = 16 \times 16 = 256$. Additional model details are presented in the Appendix \ref{sec:facerestorationdetailsAppendix}. The belief vector can be used to measure the model's confusion (via dissonance \cite{pandey2022evidential}), and belief for each class can be used to make the code item prediction (the class prediction being the class for which the model outputs maximum belief). The vacuity represents the model's lack of confidence in the prediction and can be used to identify the model's confident predictions.}

\de{If the evidential transformer outputs a highly confident prediction (indicated by low vacuity), the model is expected to be accurate, and such predictions can be trusted. In contrast, if the evidential classifier is not confident in the codebook prediction (\eg due to the low quality of the input or insufficient knowledge of the model, then the code item selected at the decoder input is expected to be incorrect. In this case, instead of relying on the transformer's top predicted code item, we could consider the evidential model's beliefs to obtain more accurate code for the decoder. Based on this insight, we introduce a novel uncertainty-guided Top$-t$ belief-based codebook selection scheme for inference. For the decoder input positions that the evidential transformer is confident (indicated by a low vacuity $\nu \leq \nu_{\text{thr}})$, we trust the evidential model and select the code item from the codebook for which the model outputs the maximum belief. In contrast, for the decoder input positions that the evidential model is not confident (indicated by a high vacuity $\nu > \nu_{\text{thr}})$, we consider the top $t$ code items of the codebook for which the model's belief is the largest. We then consider a belief-weighted combination of the predicted codes to obtain the final code item for the decoder input.}

\de{Mathematically, given the codebook $\mathcal{C} = \{\mathbf{c}_1, \mathbf{c}_2, ... ,  \mathbf{c}_K \}$, the decoder input $\mathbf{d}_{m}$ for each position in the $M-$dimensional decoder input is obtained as:}
\begin{align} \label{eqn:topkcodeformerIdea}
    \mathbf{d}_{m} = \begin{cases}
        &\mathbf{c}_i^{\text{max}}  \quad \quad\quad
        \quad\quad\quad \text{if} \quad \nu \leq \nu_{\text{thr}} \\
        &\sum_{j = 1}^t \frac{{b_j^{\text{max}} \mathbf{c}_j^{}}^{\text{max}}}{\sum_{l = 1}^t {b_l^{\text{max}}}} \quad \quad \text{Otherwise}
    \end{cases}
\end{align}
\de{
where $\nu$ represents the vacuity predicted at $m^{\text{th}}$ decoder input position, $\mathbf{c}_i^{\text{max}}$ represents the $i^{\text{th}}$ code item in the codebook $\mathcal{C}$ such that the evidential model's belief is maximum for class $i$, $b_{1}^{\text{max}},..., b_{t}^{\text{max}}$ represent the $t$ greatest belief values among the $K$ beliefs, and $\mathbf{c}_1^{\text{max}},...\mathbf{c}_t^{\text{max}}$ are the corresponding code items of codebook with the $t$ greatest belief values. When $t =1$ or the vacuity threshold $v_{\text{thr}} = 1$, the above inference scheme simplifies to the standard evidential model. When $t = K$, the model considers all the codebook items and weights them by the predicted belief to obtain the decoder input. With $t>1$, and reasonable vacuity threshold values, the inference scheme considers multiple code items for blind face restoration and uses its belief to weight the code items. 
}

\begin{table}[h!]
\centering
\begin{tabular}{lccc}
\toprule
\textbf{Evidential Model}        & \textbf{PSNR}{$_\uparrow$} & \textbf{MSE}{$_\downarrow$} & \textbf{L1 Loss}{$_\downarrow$}\\
\midrule
CodeFormer \cite{zhou2022towards}& 21.90 &  446.82 & 13.72 \\
\midrule
Evid. CodeFormer-\texttt{ReLU} & 6.62           & 15436.89          &   102.38 \\
Evid. CodeFormer-\texttt{SELU} & 21.31          &514.82 & 14.97   \\
\rowcolor{black!10} \textbf{GRED}-\texttt{SELU}  & 21.84          &  451.93      &  13.69 \\
Evid. CodeFormer-\texttt{Softplus}     &21.17        & 528.47 & 15.41  \\
\rowcolor{black!10} \textbf{GRED}-\texttt{Softplus}&21.81  &454.03      & 13.86  \\
Evid. CodeFormer-Exp          &  21.46          & 491.01       & 14.63\\
\rowcolor{black!10} \textbf{GRED}-Exp     & 21.79          & 456.25 & 13.76\\
\midrule
\rowcolor{black!20} \textbf{GRED}-Exp$ ({t = 5}) $& $\bf{22.27}$ & $\bf{409.64}$  &$\bf{12.93}$ \\  
\rowcolor{black!20} \textbf{GRED}-Exp$ ({t = 10})$& $\bf{22.33}$ & $\bf{403.69}$ &$\bf{12.86}$ \\ 
\bottomrule
\end{tabular}
\caption{CelebA Blind Face Restoration Results}
    \label{tab:CodeFormerexperiments}
\end{table}

\de{
We now carry out blind face restoration experiments using the CelebA dataset and present the overall results in Table \ref{tab:CodeFormerexperiments}, where we consider $t$ values of  $5$ and $10$. We consider a set of metrics, including Peak Signal to Noise Ratio (PSNR), 
Mean Squared Error (MSE), and the $L_1$ loss between the generated image and the ground truth image for evaluation. \cyan{We observe that the ReLU based evidential model fails to achieve reasonable performance due to its sub-optimal learning (as indicated by low PSNR, and high MSE and L1 loss values in Table \ref{tab:CodeFormerexperiments}). The evidential codeformer models with proposed correct evidence regularization (\ie the GRED variants for SELU-based , softplus-based, and exponential-based evidential models) consistently improve compared to the corresponding evidential codeformer models as the GRED regularization enables the evidential models to learn from all the training samples}. \cyan{Moreover, using the novel uncertainty-guided Top-$k$ strategy leads to significant improvements in terms of PSNR (with a boost of around 0.43 db), $L_1$ loss (decrease of around 0.42 units), and  MSE (decrease by around 43 units)}. We carry out additional ablations to show the superiority of using belief weighting compared to uniform weighting, the impact of the $t$ value, and the vacuity threshold in the Appendix. 
}

\subsection{Extension to Few-Shot Classification}

Few-shot learning operates in a regime of extremely limited supervision, making uncertainty awareness particularly important for trustworthiness and robustness \cite{vinyals2016matching,finn2017model}. Prior works have explored evidential methods for uncertainty-aware few-shot learning \cite{pandey2022evidential,pandey2022multidimensional}, but they suffer from the same zero-evidence learning limitations identified in Section~III. We improve upon these approaches by incorporating our proposed correct-evidence regularization into a modern few-shot framework. We build an evidential Visual Prompt Tuning (VPT) model by modifying the transformer classifier in VPT \cite{jia2022visual} to output evidential activations instead of softmax probabilities. The model is trained using the full evidential objective in Eq.~\ref{eqn:proposedEvidentialModelOverallLoss} with incorrect evidence regularization strength $\lambda_1 = 1.0$, and we evaluate both with and without the proposed correct-evidence regularization across multiple activation functions. Additional implementation details, along with results for other $\lambda_1$ values, are provided in Appendix~\ref{sec:FewShotLearningSettingDetails}.

Table~\ref{tab:evidentialMetalearningOverall} summarizes performance on the challenging 100-way 1-shot and 100-way 5-shot CIFAR-100 benchmarks. Standard evidential models perform poorly in these extreme low-data settings. In contrast, the proposed GRED variants consistently outperform their non-regularized counterparts across all activation functions, demonstrating improved generalization from limited supervision. Importantly, this improvement comes with \emph{no additional computational overhead} and naturally yields high-quality uncertainty estimates.

The benefit of uncertainty is further illustrated in the accuracy–vacuity curves in Fig.~\ref{fig:acc-vac-fewshot}. As the vacuity threshold increases, we retain only the model’s most confident predictions, and accuracy rises sharply. For example, at a vacuity threshold of $0.6$, the accuracy in the 100-way 1-shot setting improves from roughly $50\%$ to above $90\%$. This highlights the practical advantage of evidential uncertainty for filtering reliable predictions in few-shot scenarios.

\begin{table}[h!]
\centering
\begin{tabular}{lcc}
\toprule
\textbf{Evidential Model}        & \textbf{100-Way 1-Shot} & \textbf{100-way 5-Shot} \\
\midrule
\texttt{ReLU}         & $1.00 _{\pm 0.00} $           & $1.00 _{\pm 0.00} $              \\
\texttt{SELU}         & $36.81 _{\pm 5.69} $          & $49.39 _{\pm 3.51 }$          \\
\rowcolor{black!10} \textbf{GRED}-\texttt{SELU}    & $45.88 _{\pm 1.31} $          &  $76.09 _{\pm 0.78} $          \\

\texttt{Softplus}     &$38.85 _{\pm 3.15} $          & $47.87 _{\pm 4.78} $           \\
\rowcolor{black!10} \textbf{GRED}-\texttt{Softplus}& $\bf46.63 _{\pm 1.59} $  & $74.08 _{\pm 1.61} $         \\

Exp          &  $37.49 _{\pm 3.18} $           & $49.05 _{\pm 5.51} $         \\
\rowcolor{black!10} \textbf{GRED}-Exp     & $\bf46.54_{\pm 1.45} $          & \textbf{$\bf77.00_{\pm 1.01 }$}  \\
\bottomrule
\end{tabular}
\caption{Few-Shot CIFAR-100 Classification Results} 
\label{tab:evidentialMetalearningOverall}
\end{table}

\begin{figure}[ht!] 
    \centering
    \begin{subfigure}[b]{0.46\linewidth}
      \includegraphics[width=\linewidth]{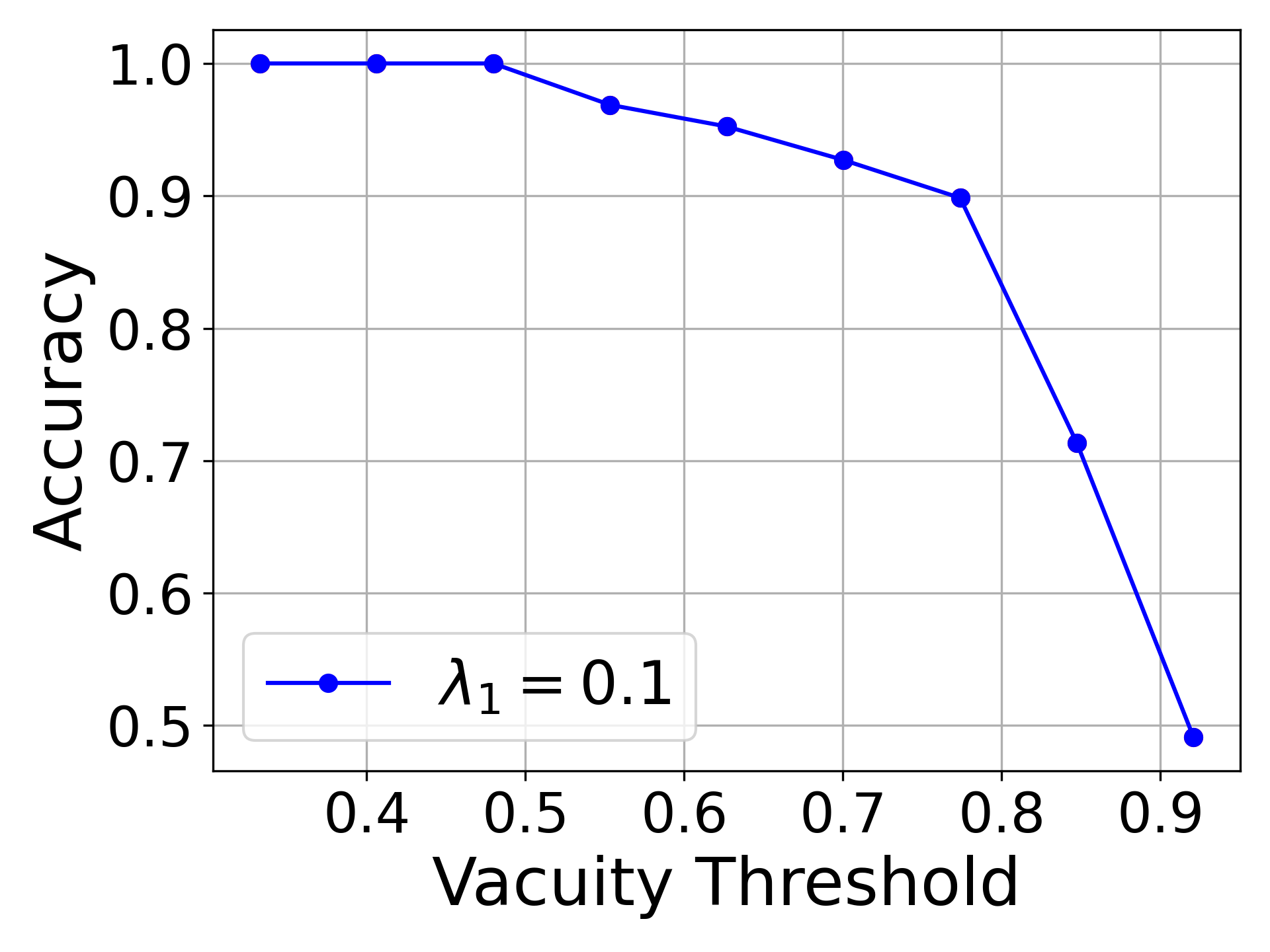}
      \caption{$100$-way $1$-shot trend}
    \end{subfigure}
    \begin{subfigure}[b]{0.46\linewidth}
      \includegraphics[width=\linewidth]{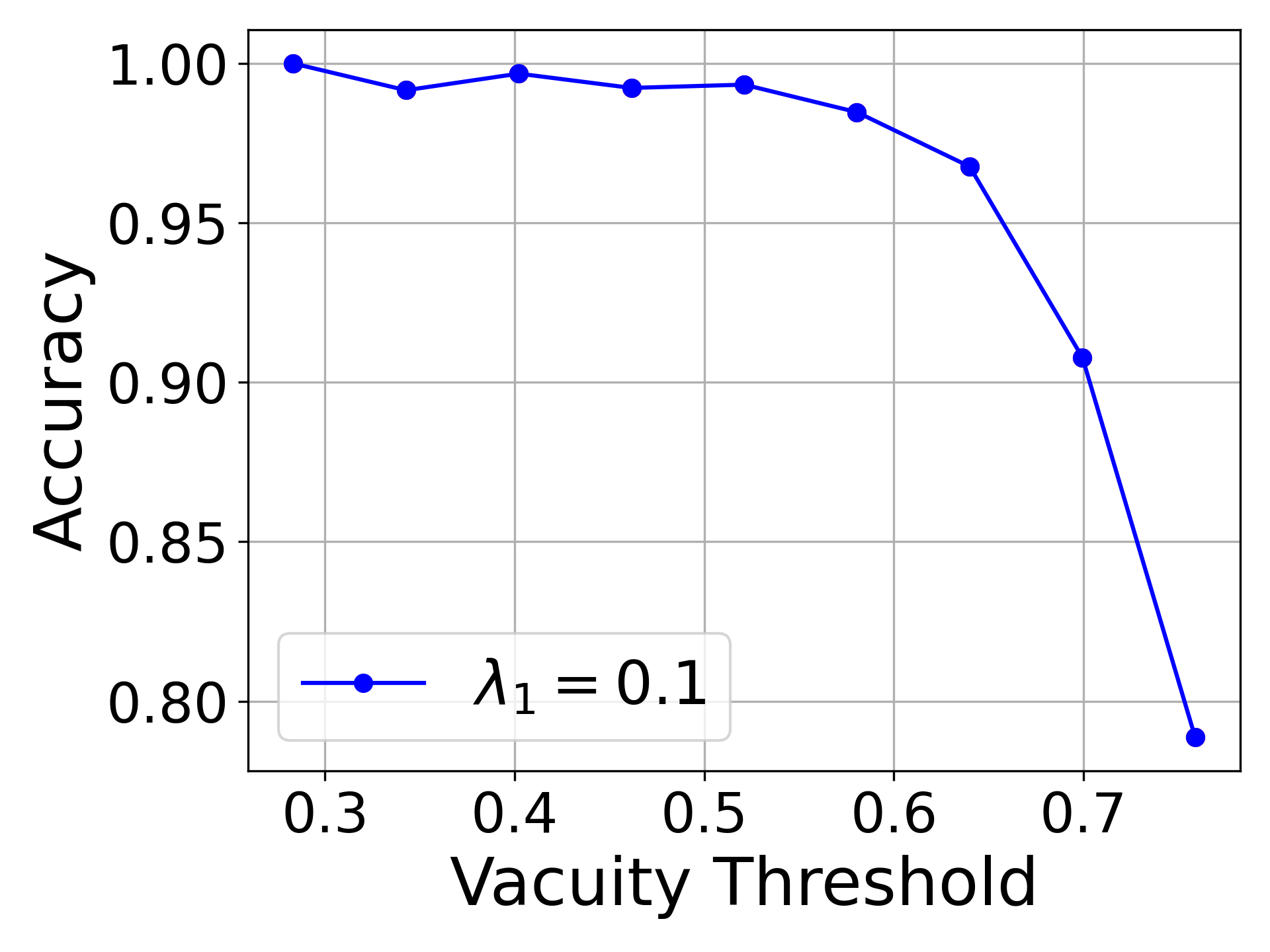}
    \caption{$100$-way $5$-shot trend}
    \end{subfigure}
\caption{Few-Shot CIFAR-100 Accuracy-Vacuity curves}
\label{fig:acc-vac-fewshot}
\vspace{-5mm}
\end{figure} 
\bl{\section{Out-of-Distribution Detection}}
\label{subsec:ood_detection}

\bl{We evaluate the ability of evidential models to assign high epistemic uncertainty to out-of-distribution (OOD) inputs. Following standard practice, CIFAR-100 test samples serve as in-distribution (ID) and SVHN as OOD. We fine-tune the few-shot models on 1-shot and 5-shot CIFAR-100 with three KL regularization strengths (0.0, 1.0, 100.0), using the uncertainty score $1 - \max p(y)$ for ID–OOD discrimination. AUROC is used as the evaluation metric.}

\bl{Figure~\ref{fig:ood_auroc} shows AUROC results across all settings. GRED consistently improves OOD separability, with the largest gains under moderate KL regularization. In the 5-shot scenario with KL$=1.0$, GRED boosts AUROC from $0.633$ to $0.882$, demonstrating substantially improved epistemic uncertainty modeling. Large KL values (e.g., 100.0) yield diminishing returns, but GRED remains superior to the baseline. Overall, correct evidence regularization is crucial for robust OOD behavior in few-shot regimes.}

\bl{
\begin{figure}[ht!]
\centering
\begin{subfigure}[b]{0.46\linewidth}
  \includegraphics[width=\linewidth]{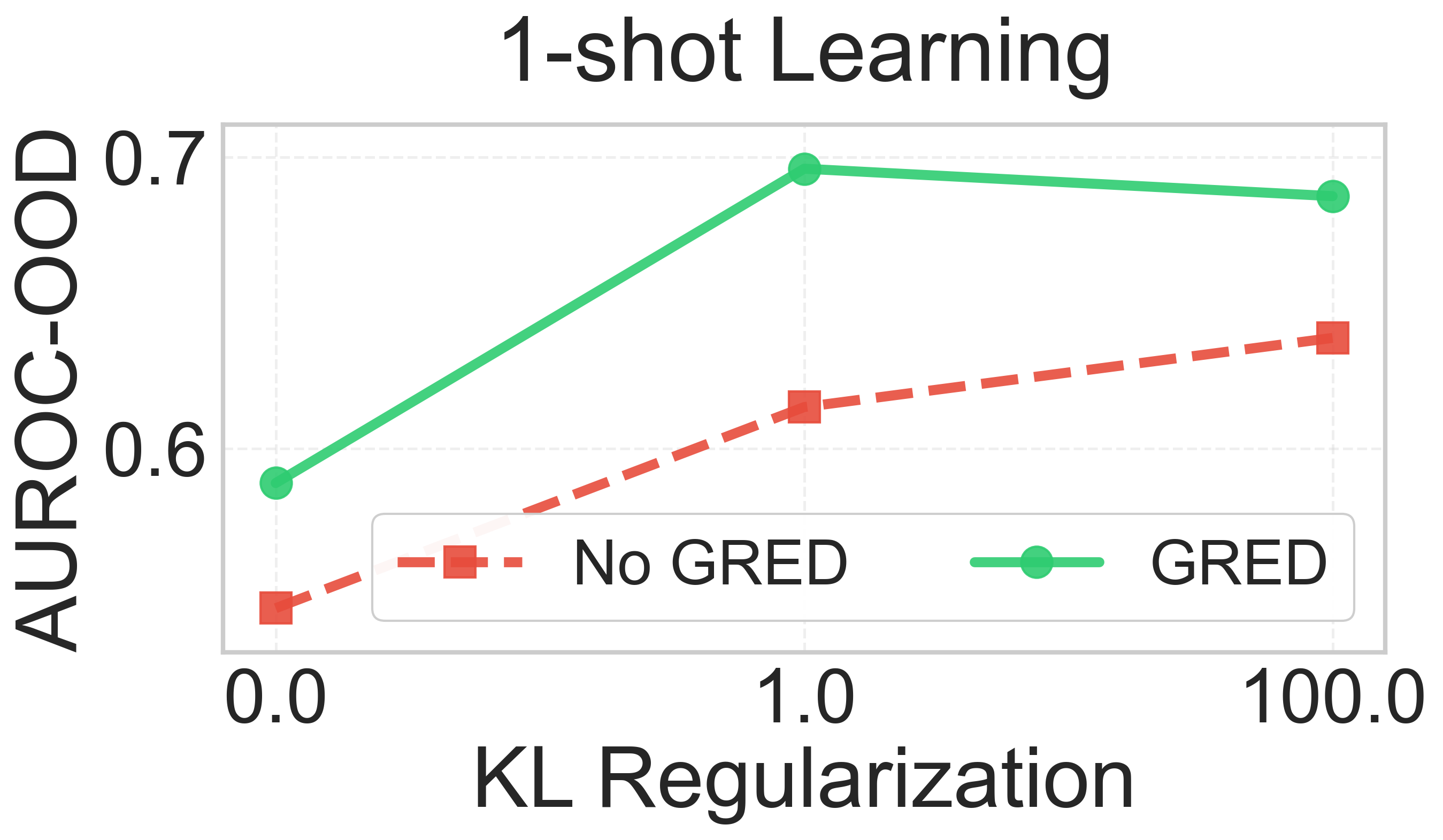}
  \caption{1-shot learning}
\end{subfigure}
\begin{subfigure}[b]{0.46\linewidth}
  \includegraphics[width=\linewidth]{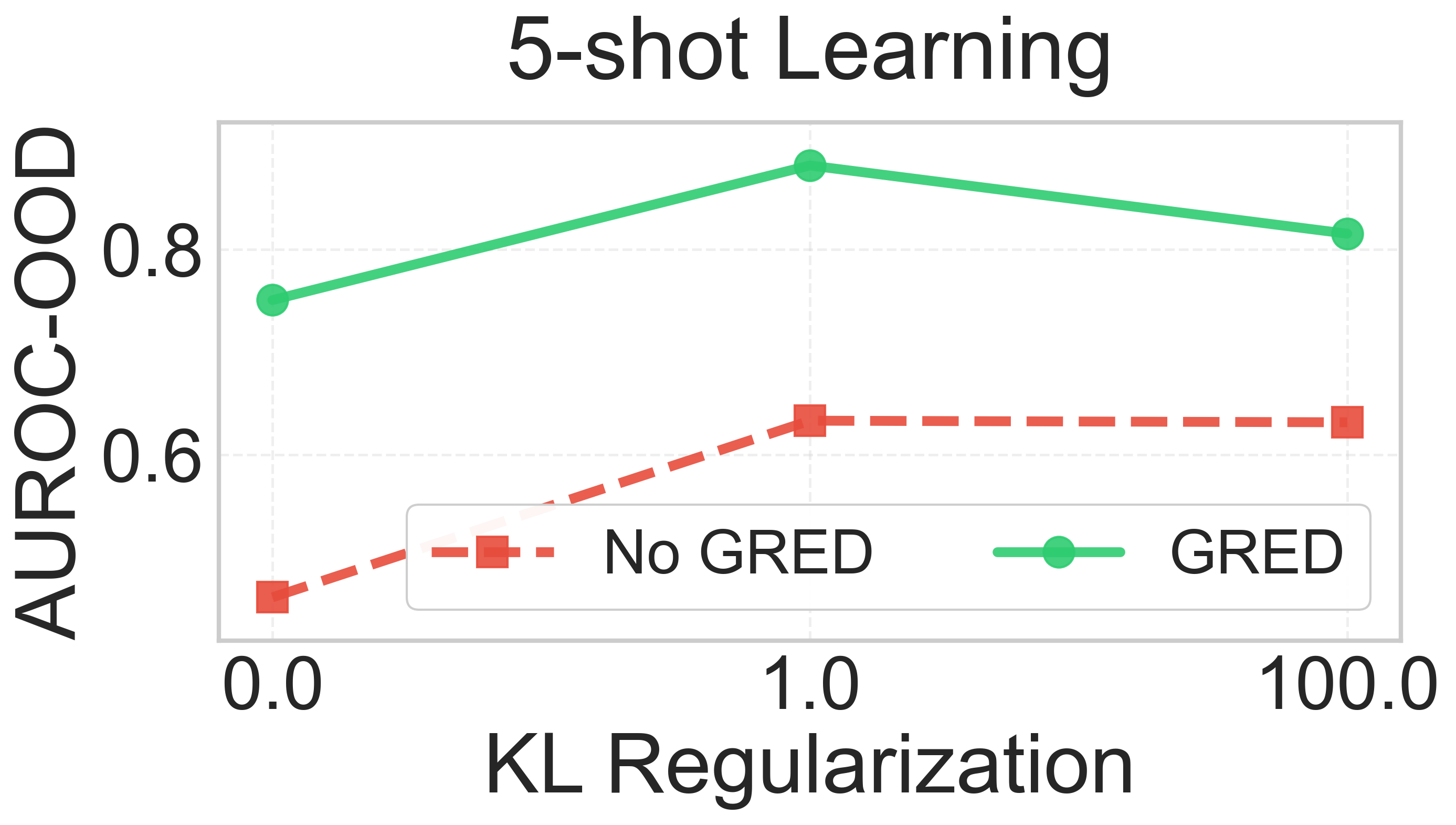}
  \caption{5-shot learning}
\end{subfigure}
\caption{AUROC-OOD detection trends. GRED consistently improves OOD separability across KL strengths.}
\label{fig:ood_auroc}
\end{figure}
}

\nocite{knopp1996weierstrass} \nocite{kingma2014adam}

\section{Conclusion}

\bl{In this work, we presented a theoretical analysis revealing a fundamental limitation of evidential deep learning models: their gradients vanish in the zero-evidence region, preventing the model from learning from precisely the samples where supervision is most needed. We showed that exponential activations provide stronger learning signals in this regime and introduced a correct-evidence regularization term that restores meaningful gradients for low- and zero-evidence samples. This yields GRED, a generalized regularized evidential model that learns from all training examples. Extensive experiments across classification, few-shot learning, adversarial evaluation, OOD detection, and blind face restoration demonstrate consistent gains in generalization and uncertainty reliability. GRED mitigates activation-induced learning freeze and advances the development of trustworthy, uncertainty-aware neural networks.}

\bibliography{ref}
\bibliographystyle{ieeetr}
\begin{IEEEbiography}
[{\includegraphics[width=1in,height=1.25in, clip,keepaspectratio]{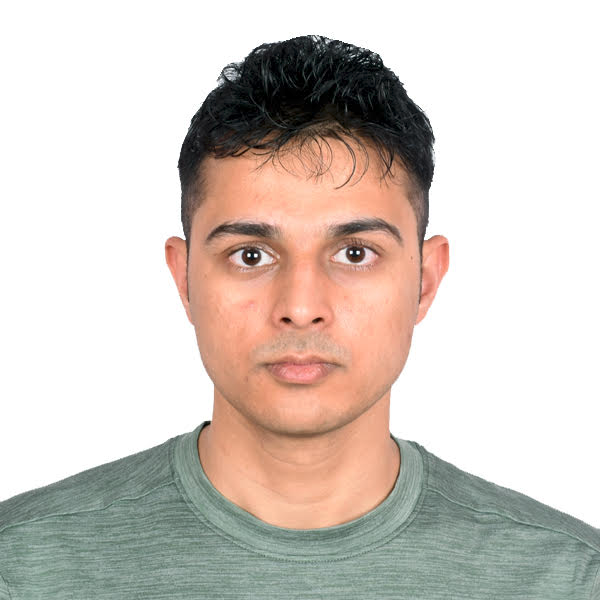}}]{Deep Shankar Pandey} received his Bachelor's degree in Electronics and Communication Engineering from Institute of Engineering, Pulchowk Campus, Tribhuvan University, Nepal in 2017. He completed his PhD degree in Computing and Information Sciences at Rochester Institute of Technology with  Dr. Qi Yu as his PhD advisor in 2025 where he worked on Uncertainty Aware Meta Learning for Learning from Limited Data. He is currently working as Applied Scientist at Amazon, and this work was done prior to joining Amazon.  He has authored several publications in top-tier venues, including NeurIPS, CVPR, ICML, and AAAI. His research focuses on developing trustworthy uncertainty-aware machine learning models with an emphasis on real-world applications.
\end{IEEEbiography}

\begin{IEEEbiography}
[{\includegraphics[width=1in,height=1.25in, clip,keepaspectratio]{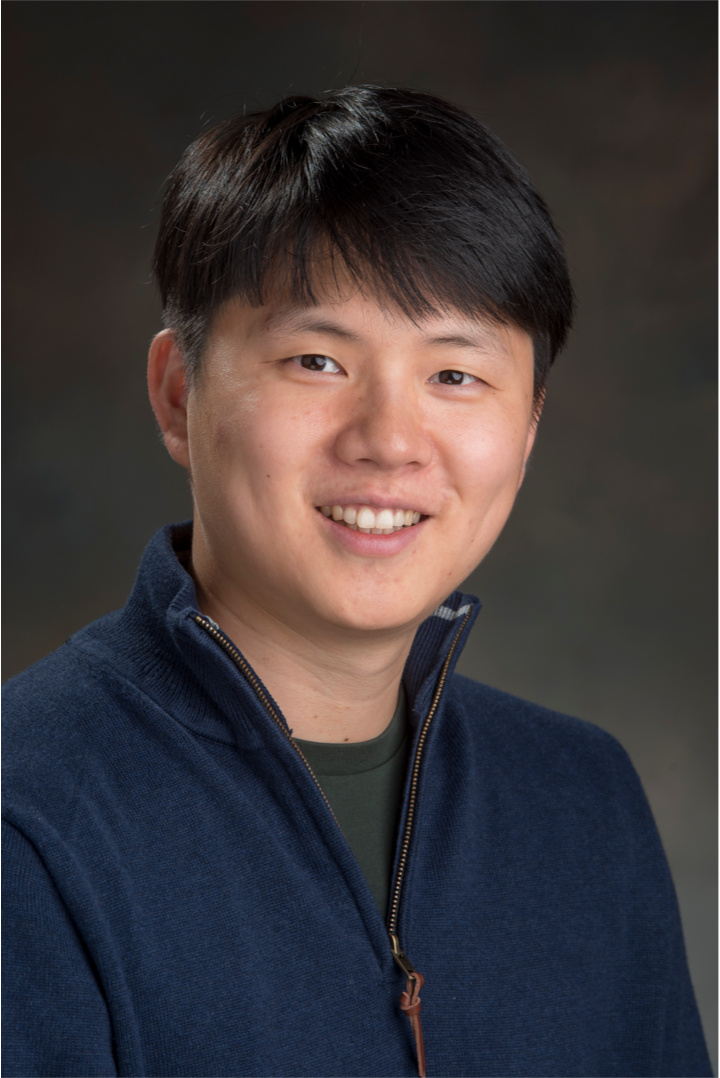}}]{Hyomin Choi} received the Ph.D. degree in engineering science from Simon Fraser University, Burnaby, BC, Canada, in 2022. He is a senior staff engineer at the AI Lab, InterDigital, in Los Altos, CA, USA. He was a research engineer at the System IC Research Center, LG Electronics, Seoul, Korea between 2012 and 2016. He received the 2017 Vanier Canada Graduate Scholarship, the 2023 Governor General's Gold Medal from Simon Fraser University, and the 2023 IEEE Transactions on Circuits and Systems for Video Technology Best Paper Award. He is currently a Member of the IEEE-CAS Multimedia Systems and Applications Technical Committee. His research interests encompass end-to-end learning-based image/video coding, video coding for machines, and machine learning with applications in multimedia processing.
\end{IEEEbiography}

\begin{IEEEbiography}
[{\includegraphics[width=1in,height=1.25in, clip,keepaspectratio]{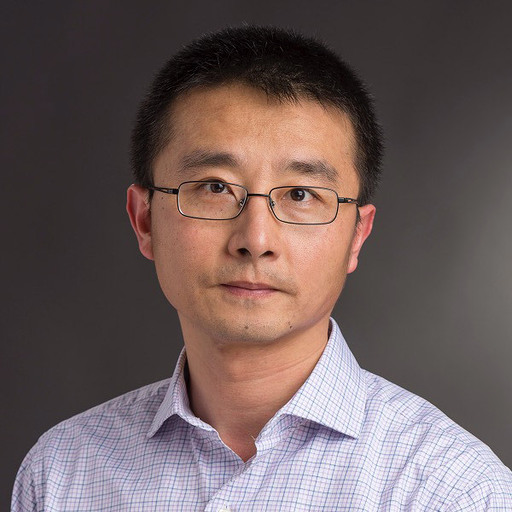}}]{Qi Yu} is a Professor in the School of Information (iSchool) at Rochester Institute of Technology (RIT). He earned his PhD in Computer Science from Virginia Polytechnic Institute and State University (Virginia Tech). Dr. Yu’s primary research interests are in artificial intelligence (AI) and machine learning (ML). He has authored over 150 publications, with many appearing in top-tier venues, including NeurIPS, ICML, ICLR, AAAI, IJCAI, AISTATS, CVPR, ICCV, and ECCV. Dr. Yu actively contributes to the research community as an area chair or senior program committee member for major conferences in AI, ML, and computer vision. Additionally, he serves as an Associate Editor for the IEEE Transactions on Services Computing and the IEEE Transactions on Cognitive and Developmental Systems.
\end{IEEEbiography}

\end{document}